\documentclass[aos]{imsart}
\usepackage{amsfonts}
\usepackage{color}
\RequirePackage[OT1]{fontenc}
\RequirePackage{amsthm,amsmath}
\RequirePackage[numbers]{natbib}
\RequirePackage[colorlinks,citecolor=blue,urlcolor=blue]{hyperref}

\usepackage{graphicx}
\usepackage{float}
\usepackage{booktabs}
\usepackage{geometry}
\usepackage{caption}
\usepackage{subcaption}
\usepackage{booktabs}
 \usepackage{verbatim}
\usepackage{algorithm}
\usepackage{algpseudocode}
\usepackage{amsmath}
\usepackage{graphics}
\usepackage{epsfig}

\startlocaldefs
\numberwithin{equation}{section}
\theoremstyle{plain}
\newtheorem{theorem}{Theorem}

\newtheorem{corollary}[theorem]{Corollary}

\newtheorem{lemma}[theorem]{Lemma}

\newtheorem{remark}[theorem]{Remark}

\endlocaldefs
\usepackage{glossaries}
\makeglossaries
\begin{document}

\begin{frontmatter}
\title{Semi-supervised learning in unbalanced and heterogeneous networks}
\runtitle{Semi-supervised learning in networks}

\begin{aug}
\author{\fnms{Ting} \snm{Li}\thanksref{m1}\ead[label=e1]{tlial@ust.hk}}
,
\author{\fnms{Ningchen} \snm{Ying}\thanksref{m1}\ead[label=e2]{nying@ust.hk}}
,
\author{\fnms{Xianshi} \snm{Yu}\thanksref{m1}
\ead[label=e3]{xyuai@ust.hk}}
\and
\author{\fnms{Bingyi} \snm{Jing}\thanksref{m1}\ead[label=e4]{majing@ust.hk}\ead[label=u1,url]{http://www.math.ust.hk/~majing/}}
\runauthor{Li, Ying, Yu and Jing}

\affiliation{Hong Kong University of Science and Technology\thanksmark{m1}}

\address{Ting Li, Ningchen Ying, Xianshi Yu, Bingyi Jing \\Department of Mathematics, \\the Hong Kong University of Science and Technology, \\Clear Water Bay, Hong Kong.\\
\printead{e1,e2,e3,e4}\\
\phantom{E-mail:\ }}
\end{aug}

\begin{abstract}
Community detection was a hot topic on network analysis, where the main aim is to perform unsupervised learning or clustering in networks. Recently, semi-supervised learning has received increasing attention among researchers. In this paper, we propose a new algorithm, called weighted inverse Laplacian (WIL),  for predicting labels in partially labeled networks. The idea comes from the first hitting time in random walk, and it also has nice explanations both in information propagation and the regularization framework. We propose a partially labeled degree-corrected block model (pDCBM) to describe the generation of partially labeled networks. We show that WIL ensures the misclassification rate is of order $O(\frac{1}{d})$  for the pDCBM with average degree $d=\Omega(\log n),$ and that it can handle situations with greater unbalanced than traditional Laplacian methods. WIL outperforms other state-of-the-art methods in most of our simulations and real datasets, especially in unbalanced networks and heterogeneous networks.
\end{abstract}

\begin{keyword}[class=MSC]
\kwd[Primary ]{62H30}
\kwd[; secondary ]{62G20}
\end{keyword}

\begin{keyword}
\kwd{semi-supervised learning}
\kwd{network data}
\kwd{heterogeneous networks}
\kwd{unbalanced networks}
\end{keyword}

\end{frontmatter}
\section{Introduction}


Network community detection is a traditional problem in network data analysis. However, in datasets from the real world, additional side information is often available. For instance, we might know some of the node memberships. How to obtain more accurate predictions under this semi-supervised situation is an interesting problem. The network-based semi-supervised learning (NSSL) discussed here is a special semi-supervised learning method that deals with network data in particular. Given the network structure and some of the labels, we would like to predict the unknown labels. NSSL has many real-life applications, for instance, for inferring unknown profiles from a social network; predicting a research topic from the co-authorship network; performing function annotation on protein or gene interaction networks; and predicting political election results. Figure \ref{fig:toy} shows a toy example of the function association in a protein-protein interaction network. 

\begin{figure}[H]
\centering
\captionsetup{justification=centering}
\begin{subfigure}[t]{0.5\textwidth}
  \centering
  \includegraphics[width=.6\linewidth]{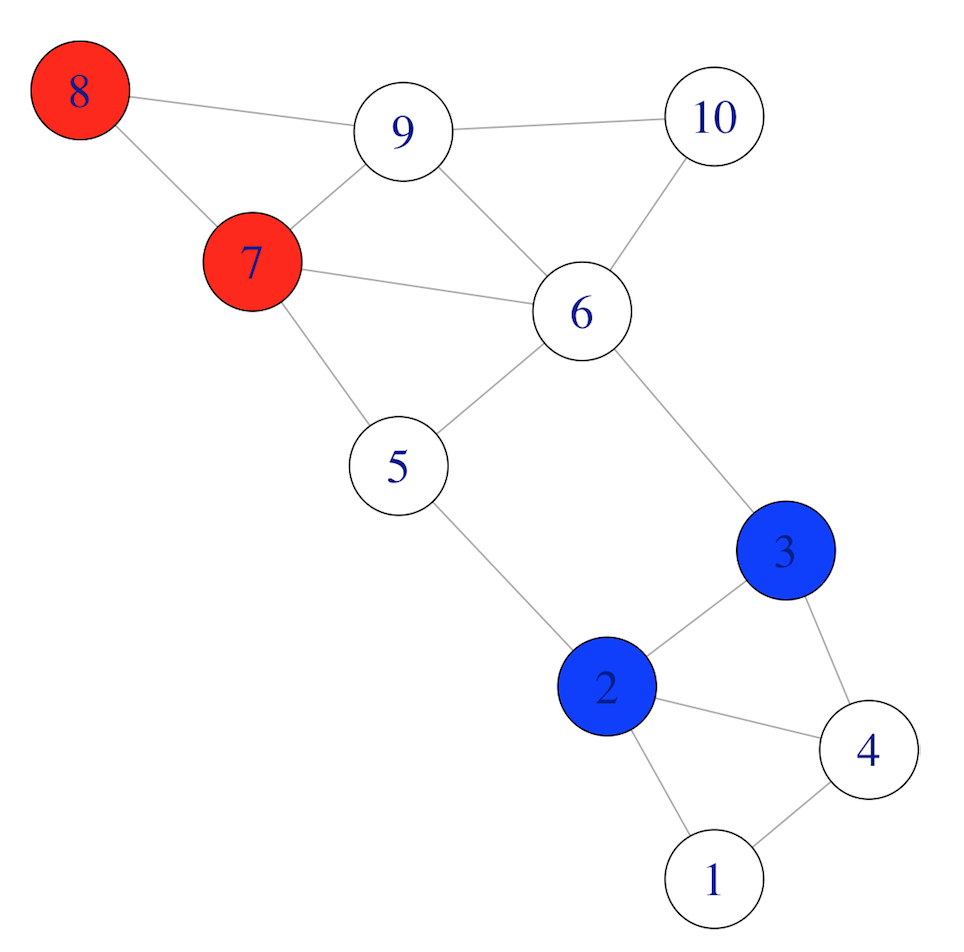}
  \caption{}
  \label{fig:input}
\end{subfigure}%
\begin{subfigure}[t]{0.5\textwidth}
  \centering
  \includegraphics[width=.6\linewidth]{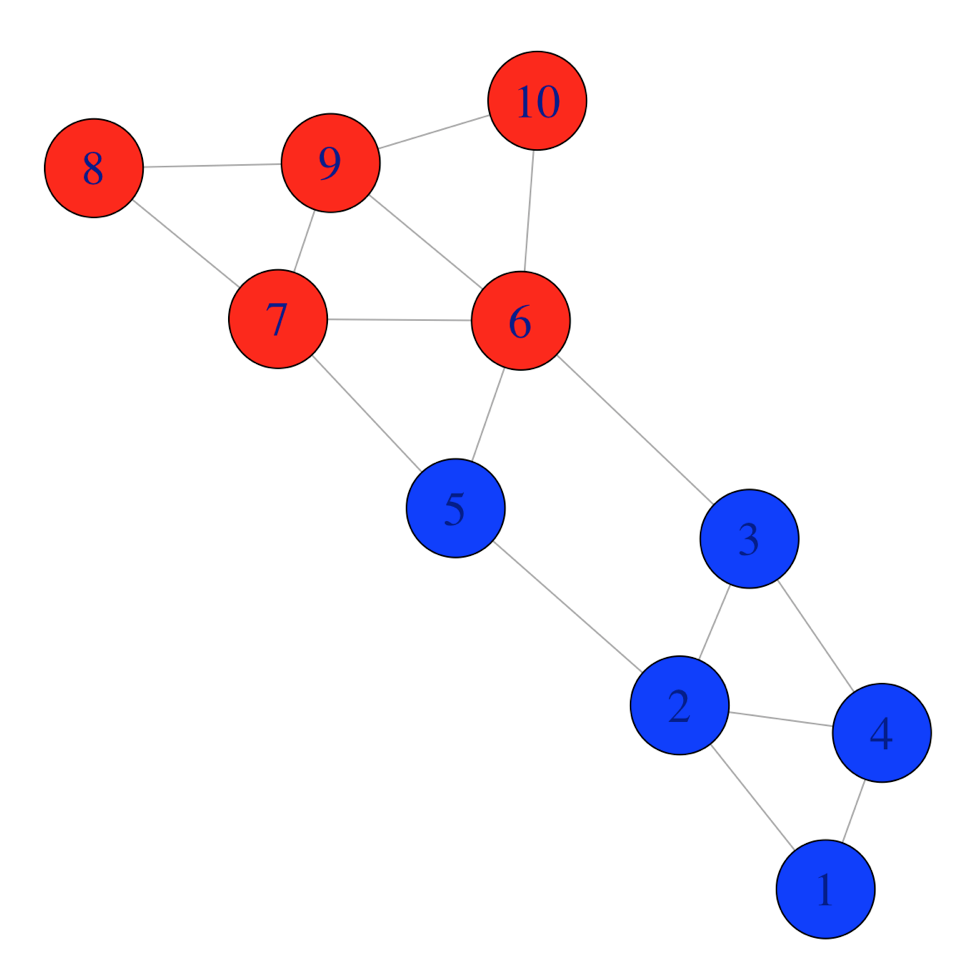}
  \caption{}
  \label{fig:output}
\end{subfigure}
\caption{A toy example of network-based semi-supervised learning: Input data consist of a protein-protein interaction network and the function labels of four proteins (2 and 3 share the same function, while 7 and 8 share another function). After applying the learning algorithm, we can predict the functions of the unknown proteins (assign 1, 4 and 5 to the blue class, and 6, 9, and 10 to the red one).}
\label{fig:toy}
\end{figure}

Most of the algorithms for community detection cannot be applied directly to NSSL. Researchers have recently attempted to find efficient algorithms for NSSL and to analyze their statistical properties, recently. \cite{r615} discussed phase transitions in the semi-supervised clustering of sparse networks using belief propagation. The effect of the relevant labels for a  vanishingly small fraction of nodes was discussed in \cite{r611} by coupling the labeled network to $k$-label broadcasting processes. \cite{r604, r610} used localized belief propagation to make predictions. The partially labeled stochastic block model  (p-SBM) was used  to model the partially labeled network and the consistency was also showed. In addition, \cite{r616} proposed two spectral-based methods that mainly focused on assortative networks (in which nodes with different labels are more likely to be connected to each other). In \cite{r612}, the authors proposed linearized belief propagation with a novel weighted initialization called Weighted Message Passing (WMP) to perform clustering in partially labeled networks generated from SBM. Moreover, a confidence-aware algorithm called CAMLP was proposed in \cite{camlp} to tackle both homophily and heterophily networks at the same time.

In contrast to previous studies, this paper not only tackles the general NSSL problem, but also pays particular attention to unbalanced and heterogeneous networks. We propose an effective algorithm to solve the NSSL problem. We also introduce a generative model to describe the data and then determine the consistency of our new algorithm under the model. 

\subsection*{Unbalanced networks}
Imbalance is widely observed in real-world networks. For example,  in gene interaction networks from the Saccharomyces Genome Database (SGD) \cite{r602}, the number of viable genes can be four times that of inviable genes, which makes the dataset highly unbalanced. Networks in our analysis of real datasets analysis are also unbalanced according to the labels. When performing semi-supervised learning in network data, imbalance will cause significant bias if we do not take any action to rectify it. Labels (communities) with more members might be able to absorb more unlabeled nodes. Such bias should be dealt with in order to identify nodes in small communities. All of these problems are seldom discussed in the NSSL literature. In the present paper, we tackle the issue of network imbalance with our new algorithm by normalizing the weight of nodes.

\subsection*{Heterogeneous networks}
A commonly used model for networks with community structures is the SBM, first presented by \cite{r603}. For decades, SBM has raised research interest in computer science, statistics, business studies as well as physics. Algorithms and consistency associated with community detection in SBM have been studied extensively. See, for example, \cite{r1, r313, r310, r311, r5, r9, r13}.  However, due to the assumption of SBM, the nodes within the same community have the same degree ditribution, which is not observied in real-world datasets. Nodes in real networks often show degree heterogeneity even they have the same label (within the same community). Examples can be found in our real-data analysis. In order to accommodate hubs in networks, Karrer and Newman proposed the degree-corrected stochastic block model (DCBM) in \cite{r6}. The theoretical property of DCBM has been studied for community detection problem in \cite{r4,r9, r600, r601}. However, to the best of our knowledge, we are the first to apply the DCBM to the NSSL problem. In the present paper, we not only study the NSSL problem in homogeneous networks but also in heterogeneous ones. We propose the partially labeled DCBM (pDCBM) to model the generation of data and prove the consistency of our new algorithm.

\subsection{Our contributions}
We summarize the main results of this paper as follows:

\subsubsection*{Weighted inverse Laplacian algorithm}
A new semi-supervised learning algorithm called the weighted inverse Laplacian (WIL) algorithm is proposed for solving the NSSL problem. By integrating the global information in the network with different normalizations of the adjacency matrix, the WIL algorithm is designed to eliminate heterogeneity and imbalance issues. With a simple form, the WIL algorithm has explanations in different points of views including information propagation, the regularization framework ant the first hitting time. It also enjoys statistical guarantee with a consistency rate in the order of $O(\frac{1}{degree}).$ Both simulation and real-data analysis show the advantage of the WIL algorithm in most of the test scenarios.

\subsubsection*{Partially labeled DCBM}
We propose a generative model called the partially labeled DCBM (pDCBM). Based on the DCBM, and by introducing the popularity of nodes, the pDCBM describes a more general data structure than the p-SBM mentioned in \cite{r604}. Theoretical study is also carried out under the pDCBM setting.

\subsubsection*{Statistical guarantee}
Theoretically, we prove the consistency of our new algorithm for the NSSL problem under the pDCBM and explore the effect of the unbalanced ratio, the out-in ratio and the labeled ratio. Our main result is as follows:
$$\mathbb{P}(err \geq \epsilon )\leq \frac{c}{\epsilon ^2(1-\beta)^2s^2\delta^2d}$$ 
which shows the convergence rate as $O(\frac{1}{d})$, the inverse of the average degree. The convergence rate also decreases with the unbalanced ratio $s$ and the labeled ratio $\delta$, but increases with the out-in ratio $\beta,$ which makes sense and matches results in our empirical study .

\subsubsection*{Transition Boundary}
We discuss the transition boundary of guaranteed consistency on the unbalanced ratio and the out-in ratio in the pDCBM. We propose that by taking particular parameters in WIL, we can tackle networks that are more unbalanced  than traditional Laplacian methods (e.g. random walk and normalized Laplacian). More details can be found in Theorem \ref{boundary} below.

\subsection{Organization of the paper}
We organize the rest of the paper as follows. We propose a new algorithm called WIL for the NSSL problem in Section \ref{sect_method}, and explain it from the angles of information propagation, the regularization framework and first hitting time in random walks.  In Section \ref{sect_consistency}, we first propose a generative model, the partially labeled DCBM (pDCBM) to describe the NSSL problem. Statistical guarantee and phase boundary are also discussed in Section \ref{sect_consistency} under the pDCBM framework. In Section \ref{sect_simu} and Section \ref{sect_real}, we show the numerical results of our new method and cutting-edge methods using both simulation data and real-world networks. We review and summarize related work in Section \ref{sect_related_works}. Finally, in Section \ref{sect_CONC}, we conclude the paper and suggest directions for future work. The technical proof of main results is given in the Appendix.

\section{Methodology}  \label{sect_method}
We propose an algorithm that utilizes the global connection information by calculating the sum of different powers of the normalized adjacency matrix in this section. We call it the weighted inverse Laplacian (WIL) method. Detailed explanations of the WIL algorithm are also presented in this section. 

First of all, we define the NSSL problem mathematically and introduce notation for later use. A network is represented as a graph $G=\{V, E\}, |V|=n,$ where $V$ is the set of nodes and $E$ is the collection of edges. Each node $i \in \{1,\dots,n\}$ is assigned a class label $y_i\in \{1,\dots,K\},$ but we only observe these for a subset of nodes $L,$ and the set of remaining nodes is $U=V/L.$ Our aim is to find the class labels of nodes in $U.$ Let $A$ be the adjacency matrix of $G$, and for any element of $A$
\[ a_{i,j}=
  \begin{cases}
    1       & \quad e_{i,j}\in E\\
    0  & \quad \text{ else.}
  \end{cases}
\]
$\hat{D}$ denotes the degree matrix of $A$ with $\hat{d}_{i,i}=\hat{d_{i}}=\sum_{j=1}^{n}a_{i,j}$ and the off-diagonal entries are all $0.$ $Y$ is an $n\times K$ matrix that encodes the given labels:
\[ Y_{i,c} =
  \begin{cases}
    1       & \quad \text{if node}\  i\in L \text{ and}\  y_i=c\\
    0  & \quad \text{else. }
  \end{cases}
\]
\subsection{WIL Algorithms}
We will first present the algorithm and then show some its explanations. We obtain $F$, the matrix of the class label scores, by applying the WIL algorithm:
$F=(\rho W+(1-\rho)W^T)Y,$ where $W=(I-\alpha\hat{D}^{-1}A)^{-1}$, and both $\alpha$ and $\rho$ are positive parameters less than 1. To understand WIL, we can write $W=\sum_{i=0}^{\infty} (\alpha \hat{D}^{-1}A)^i,$ so that WIL combines the global link information with exponential decade random work and eliminates the effect of hubs by dividing $A$ by $\hat{D}.$ By performing global integration and normalizing degrees, WIL can overcome the issue of imbalance and heterogeneity in networks. The algorithm can be summarized as follows:

\begin{algorithm}[htb]
  \caption{ Weighted inverse Laplacian algorithm (WIL)}
  \label{alg:WIL}
  \begin{algorithmic}[1]
    \Require
     Adjacency matrix: $A_{(n\times n)}$; 
     Known labels matrix: $Y_{(n\times K)}$;
     Parameters: $\alpha$, $\rho$
    \Ensure
      Label score matrix: $F_{(n\times K)}$ 
    \State Calculate the degree matrix: $\hat{D}=diag(A\textbf{1}\textbf{1}^{T})$;
    \label{code:fram:degree}
    \State Add up the powers of normalized matrix: $\hat{W}=(I-\alpha \hat{D}^{-1}A)^{-1}$;
    \label{code:fram:add}
    \State Combination: $\hat{M}=\rho \hat{W}+(1-\rho)\hat{W}^T$;
    \label{code:fram:combination}
    \State Get label score matrix: $F=\hat{M}Y$;
    \label{code:fram:score} \\
    \Return $F$;
  \end{algorithmic}
\end{algorithm}

\begin{remark}
Regarding Algorithm \ref{alg:WIL}:
\begin{itemize}
\item
The diagonal entries of $A$ and $\hat{M}$ are set to 0 in order to void self-reinforcement.

\item
When handling a large network in practice, it is adequate to approximate $\hat{W}$ by $\sum_{k=1}^{m}\alpha^k(\hat{D}^{-1}A)^k$, rather than to calculate $\hat{W}$. 

\item
For the hard classification problem, we pick the column index of the maximum score in vector $F_{i,*}$ as the label of node $i \in U.$ We remain the value of vector $F_{i,*}$ as the relative probability for mixed membership setting.
\end{itemize}
\end{remark}

\begin{remark}
Two tuning parameters, $\rho$ and $\alpha$, must be chosen. 
\begin{itemize}
\item
$\rho \in [0,1]$ is a weight parameter that can be chosen from labeled  data without adding too much computational load. More details will be given in the simulation and real-data analysis below. 
 
\item 
We recommend setting $\alpha=e^{-0.25}$. This has been found to be a good choice for a host of scenarios. Although $\alpha$ can also be tuned with training data, we find that WIL is robust to $\alpha.$ However, if one has enough training data and time consuming is not a major concern, then $\alpha$ should be tuned via cross-validation.
\end{itemize}
\end{remark}

\subsection{Derivation of WIL from information propagation}

The formulation of WIL is motivated by the idea of information propagation. WIL is a combination of two different kinds of information propagation processes. 

First, we consider $\hat{F}=\hat{W}Y=(I-\alpha \hat{D}^{-1} A)^{-1}Y$, which can also be  written as an information propagation process:

\begin{algorithm}[H]
  \caption{ Information Propagation with Normalization on Targeting Points}
  \label{alg:Norm_targ}
  \begin{algorithmic}[1]
    \Require
     Adjacency matrix: $A_{(n\times n)}$; 
     Known labels matrix: $Y_{(n\times K)}$;
     Parameters: $\alpha$
    \Ensure
      Label score matrix: $\hat{F}_{(n\times K)}$ 
    \State Calculate the degree matrix: $\hat{D}=diag(A\textbf{1}\textbf{1}^{T})$;
    \label{code:fram:degree1}
    \State Iterate $\hat{F}(t+1)=\alpha \hat{D}^{-1}A\hat{F}(t)+(1-\alpha)Y$ until convergence;
    \label{code:fram:iterate1}\\
    \Return $\hat{F}$;
  \end{algorithmic}
\end{algorithm}
Algorithm \ref{alg:Norm_targ}, which is also mentioned by \cite{r608}, can be understood intuitively in terms of spreading label information with normalization on targeting points. We use adjacency matrix $A$ to spread the label, while we normalize rows of $A$ with degree of nodes. Figure \ref{fig:target} gives a toy example of one-step information propagation from node $i$'s neighbors to node $i$ in Algorithm \ref{alg:Norm_targ}.

\begin{figure}[H]
\centering
  \includegraphics[width=0.5\linewidth]{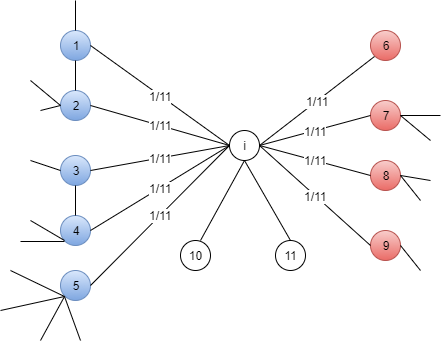}
\caption{Illustration of Algorithm \ref{alg:Norm_targ}: all edges are re-weighted as $1/11$, and after this propagation, $\hat{F}_i=(5/11, 4/11)$, which means that $i$ is more likely to be blue.}\label{fig:target}
\end{figure}

$\tilde{F}=\hat{W}^TY=(I-\alpha  A\hat{D}^{-1})^{-1}Y$ can be written aster another information propagation algorithm:

\begin{algorithm}[H]
  \caption{ Information Propagation with Normalization on Sourcing  Points}
  \label{alg:Norm_sou}
  \begin{algorithmic}[1]
    \Require
     Adjacency matrix: $A_{(n\times n)}$; 
     Known labels matrix: $Y_{(n\times K)}$;
     Parameters:$\alpha$
    \Ensure
      Label score matrix: $\tilde{F}_{(n\times K)}$ 
    \State Calculate the degree matrix: $\hat{D}=diag(A\textbf{1}\textbf{1}^{T})$;
    \label{code:fram:degree2}
    \State Iterate $\tilde{F}(t+1)=\alpha A \hat{D}^{-1}\tilde{F}(t)+(1-\alpha)Y$ until convergence;
    \label{code:fram:iterate2}\\
    \Return $\tilde{F}$;
  \end{algorithmic}
\end{algorithm}
The main difference in Algorithm \ref{alg:Norm_sou} is that we normalize the label information by the degree of sourcing points. A toy example of one-step information propagation is given in Figure \ref{fig:source}.

\begin{figure}[H]
\centering
  \includegraphics[width=0.5\linewidth]{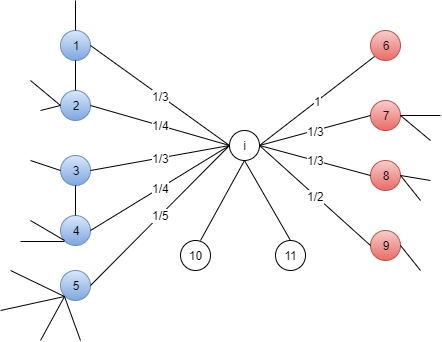}
\caption{Illustration of Algorithm \ref{alg:Norm_sou}: each edge is re-weighted by according to the degree of the node that is connected node $i$, and after this propagation, $\tilde{ F}_i=(41/30, 13/6)$, which means node $i$ is more likely to be red.}\label{fig:source}
\end{figure}

We raise examples for the above two different kinds of information propagation processes. For Algorithm \ref{alg:Norm_targ}, we take node $i$ as a student whose total social time is limited, so the influence from a certain friend would be averaged by the number of friends node $i$ has. In Algorithm \ref{alg:Norm_sou}, the social time of node $i$'s friends is also limited, so their influence should been normalized by their own degrees. Additionally, it is easy to show that bias introduced by imbalance is eliminated in some kind by the normalization in Algorithm \ref{alg:Norm_sou}. However, in the real world, it is not easy to tell exactly how does information propagate in the network. Therefore, we combine these two different kinds of information propagation in Algorithm \ref{alg:WIL} by introducing a weight parameter $\rho.$

\subsection{Derivation of WIL from the regularization framework}
Now, we develop regularization frameworks for the above two iteration algorithms. The cost function associated with $\hat{F}$ (Algorithm \ref{alg:Norm_targ}) is
\begin{equation} \label{eq:nor_target}
Q_1(F)=\frac{1}{2}(\sum_{i,j=1}^nA_{i,j}\|F_i-F_j\|^2 + \mu \|\hat{D}^{\frac{1}{2}}(F-Y)\|^2),
\end{equation}
here $\mu>0$ is a constant parameter. Set $\bar{F}=arg \min\limits_{F} Q_1(F).$  Then 
$$0=\frac{\partial Q_1 }{\partial F}\bigg|_{F=\bar{F}}=\hat{D}^{\frac{1}{2}}(\bar{F}-\frac{1}{1+\mu}\hat{D}^{-1}A\bar{F}-\frac{\mu}{1+\mu}Y).$$ 
Set $\alpha=\frac{1}{1+\mu},$ we have
$$\bar{F}=(1-\alpha)(I-\alpha \hat{D}^{-1}A)^{-1}Y,$$
which is the closed form of Algorithm \ref{alg:Norm_targ}.

Similarly, the cost function for $\tilde{F}$ (Algorithm \ref{alg:Norm_sou}) can be written as 
\begin{equation} \label{eq:nor_source}
Q_2(F)=\frac{1}{2}(\sum_{i,j=1}^nA_{i,j}\|\frac{1}{d_i} F_i-\frac{1}{d_j} F_j\|^2 + \mu \|\hat{D}^{-\frac{1}{2}}(F-Y)\|^2),
\end{equation}
here $\mu>0$ is a constant parameter. Set $\bar{F}=arg \min\limits_{F} Q_2(F).$  Then 
$$0=\frac{\partial Q_2}{\partial F}\bigg|_{F=\bar{F}}=\hat{D}^{-\frac{1}{2}}(\bar{F}-\frac{1}{1+\mu}A\hat{D}^{-1}\bar{F}-\frac{\mu}{1+\mu}Y).$$ 
Set $\alpha=\frac{1}{1+\mu},$ we have
$$\bar{F}=(1-\alpha)(I-\alpha A\hat{D}^{-1})^{-1}Y,$$
which is the closed form of Algorithm \ref{alg:Norm_sou}.

Since we have assumed nodes with the same labels are more likely to be connected, a good classifying function should not change too much between linked points. The first term in both cost functions is the smoothness constraint. However, in $Q_1$ we simply calculate the $l_2$ norm between $F_i$ and $F_j$, while we normalize $F_i$ and $F_j$ by their respective degrees respectively in $Q_2$.

The second term is the fitting constraint, which means a good classifying function should not change too much from the initial label assignment. In $Q_1$, the fitting constraint can be written as $\sum_{i=1}^{n}\hat{d}_i\|F_i-Y_i\|^2.$ As for any node $j \in L$, predicting the $j's$ label wrongly will lead to $\hat{d}_j$ punishment. This is explained by Fig. \ref{fig:target} in which node $j$ with a larger degree gives more information that is wrong . However, in $Q_2$, the second term is $\sum_{i=1}^{n}\frac{1}{\hat{d}_i}\|F_i-Y_i\|^2$, which means the wrong information is normalized by the node's degree. It can be understood from Fig \ref{fig:source} that the impact of one node on another is $\frac{1}{degree}.$

$\mu$ makes a trade-off between two competing constraints. Since $\alpha=\frac{1}{1+\mu}$, a large $\mu$ means more weight on the fitting constraint, which means a small $\alpha$ in both Algorithm \ref{alg:Norm_targ} and Algorithm \ref{alg:Norm_sou}.

\subsection{Derivation of WIL from the first hitting time}
In the NSSL problem, it is essential to properly measure the closeness or similarity between every pair of nodes. One such possible measure is the  \emph{first hitting time} between nodes, if we cast the network into the context of random walks. This is partially motivated by the fact that random walks are easily trapped within nodes having the same labels. 

Consider a random walk in a network: starting from a node, one of its edges is chosen with equal probability. Let $t_{i,j}$ denote the first hitting time from node $i$ to node $j$. Then $E(\exp(-\tau t_{i,j}))$ is a good local similarity measure between the two nodes, where the exponential transformation emphasizes the local information by down-weighting the long first hitting time. However, $E(\exp(-\tau t_{i,j}))$ is very difficult to calculate, so we approximate it by 
$$
H=\sum_{k=1}^{\infty} \exp(-\tau k)(\hat{D}^{-1}A)^k = \sum_{k=1}^{\infty} \alpha^k (\hat{D}^{-1}A)^k, \qquad \mbox{where} \quad \alpha = e^{-\tau},
$$ 
It is easy to see that $H=( I - \alpha \hat{D}^{-1} A )^{-1}-I$. In this approximation, instead of counting only the first hitting time, we count all hitting times. Since $\exp(-\tau k)$ is very small when $k$ is large, the approximation is reasonable. In addition, we notice that $H$ is an asymmetric matrix and $h_{i,j}$ stands for the impact of node $j$ on node $i$. Finally, we add up $h_{i,j}$ and $h_{j,i}$ with weights $\rho$ and $1-\rho$ respectively to measure the similarity between nodes $i$ and $j$.

\section{Main Results}\label{sect_consistency}
In this section, we first propose a generative model called the partially labeled DCBM (pDCBM). Then, under the pDCBM frame work, we show the theoretical guarantee of the WIL algorithm and its phase boundary. 

\subsection{Generative Model}\label{sect_model}
First, we introduce a prior distribution for the labels, a K-dimensional vector $\pi$ with $\sum_{i=1}^{K}\pi_i=1,$ to generate the labels. Let $B$ be the $K\times K$ symmetric probability matrix with $0\leq b_{i,j}\leq 1, \ \forall \ i,\ j \in \{1,2,\dots,K\}.$ Set membership vectors $z_i \in \{0,1\}^K,$ where $z_{i,k}=1$ indicates that node $i$ belongs to label $k$. The DCBM introduces a set of degree-corrected parameters $\{\theta_i: i=1, \dots , n\}.$ As the given labels might be incorrect, we can also introduce the parameter $\nu \in [0,1]$ to represent the probability that the given labels are correct. Let $C$ be the $K\times K$ matrix where $C_{kk}=\nu$ for all $k,$ and $C_{kl}=\frac{1-\nu}{K-1}$ for $k \neq l.$ Let $A$ be the adjacency matrix of $G.$ We can define the generation process as follows:
\begin{itemize}
\item Prior distribution of labels: $z_i \sim Mult(\cdot | \pi).$
\item For each node pair $(i,j)$, $A_{i,j} \sim Bern(\cdot |\theta_i \theta_j z_i B z_j^T).$
\item For each labeled node, $y_i \sim Mult(\cdot | z_iC).$
\end{itemize}

\begin{remark} Regarding the pDCBM:
\begin{itemize}
\item The distribution of labels follow the multi-normal distribution with $\pi.$ When $n$ is large enough, the distribution is relatively stable, so we can simply ignore the randomness in this step when performing statistical analysis.
\item $\nu$ determines the credibility of given labels. However, we analyze the consistency by setting $\nu=1.$ The proof can be easily extended to $\nu<1.$ 
\end{itemize}
\end{remark}

\subsection{Main Results}
Before presenting the main theories, we give useful notation first. For the pDCBM setting, we set $B=qE_K+(p-q)I_K,$ where $E_K$ is a $K\times K$ matrix with elements all equal to $1$ and $p>q.$ We write $c(i)=k$ when the $i's$ label is $k$ and $\hat{c}(i)=k$ when the $i's$ predicted label is $k$ by applying the WIL algorithm with proper $\alpha$ and $\rho$. We set $n_i=|\{j|j\in G \and c(j)=i\}|$ and $l_i=|\{j|j \in L \and c(j)=i\}|.$ Let the known ratio $\delta=\frac{|L|}{|G|}.$ We consider the average error rate $err=\frac{1}{(1-\delta) n}\sum_{i \in U}1_{\hat{c}(i) \neq c(i)}$ and try to prove the weak consistency: for any $\epsilon > 0$, $P(err \geq \epsilon )=f(n,\epsilon)\rightarrow 0.$

The main results of this paper are given as follows:

\begin{theorem} \label{homo}
Under the pDCBM setting, with $K=2$, $\max_i\{\theta_i\}=1$, $\min_i\{\theta_i\}\geq \epsilon_1 > 0$, where $\epsilon_1$ is a constant, $\forall \  u \in [K], \frac{1}{n_u}\sum_{c(i)=u}\theta_i\in [1-\delta_1,1],$ where $\delta_1=o(1)$, $\frac{1}{l_u}\sum_{c(i)=u, i\in L}\theta_i\in [1-\delta_1,1]$,  $d=np=\Omega (\log n) \ and \ s >  g(\beta), \forall \  constant \ \epsilon >0,$ there exist some constant $c >0, \alpha>0 $ and $\rho \geq 0.$ We have:
 
$$\mathbb{P}(err \geq \epsilon )\leq \frac{c}{\epsilon ^2(1-\beta)^2s^2\delta^2d},$$

where  $g(\beta)=\frac{1}{2\beta }((\beta-1)+\sqrt{4\beta ^3+\beta^2-2\beta +1})$, $0< \beta=\frac{q}{p} <1$ and  $0< s=min\{\frac{n_1}{n_2},\frac{n_2}{n_1}\}\leq 1.$
\end{theorem}

\begin{remark}Regarding Theorem \ref{homo}:
\begin{itemize}
\item
We can replace $\min_i\{\theta_i\}\geq \epsilon_1 > 0$ by $\min_i\{d_i\}=O(d)=\Omega (\log n).$ Both ensure that the degree is not too small for prediction.

\item
As long as $K$ is a given constant that does not tend to infinity with $n$, the result is still correct. It can be easily proved by following the proof with $K=2.$ We discuss the effect of $K$ in Appendix B.

\item
We obtain $g(\beta)=\frac{1}{2\beta }((\beta-1)+\sqrt{4\beta ^3+\beta^2-2\beta +1})$ by setting $\rho=0$ in WIL. When $\rho \in [0,1]$, $g(\beta)_{\rho}=\rho \beta + (1-\rho)\frac{1}{2\beta }((\beta-1)+\sqrt{4\beta ^3+\beta^2-2\beta +1})$ and it can be proved that $g(\beta)_{\rho}\geq g(\beta), \ \beta \in (0,1).$ 

\item 
Parameter $\alpha$ in WIL is absorbed into $c$ in the main result. Throughout the proof, we find that it is possible to estimate the optimal $\alpha.$ However, we find that WIL is quite robust against $\alpha$. Hence we recommend setting a default $\alpha.$ One can still learn $\alpha$ from training data (the labeled nodes) if time is of no concern and enough training data are available.
\end{itemize}
\end{remark}

The technical proof of Theorem \ref{homo} is given in Appendix A.

The following theorem compares the phase boundaries for the unbalanced ratios of random walk (Algorithm \ref{alg:Norm_targ}), normalized Laplacian \cite{r608} (replacing $A\hat{D}^{-1}$ in Algorithm \ref{alg:Norm_sou} by $\hat{D}^{-\frac{1}{2}} A \hat{D}^{-\frac{1}{2}}$) and Algorithm \ref{alg:Norm_sou}.

\begin{theorem}\label{boundary}
Under the pDCBM setting, with $\theta_i=1\ \forall i \in \{1,2,\dots,n\},\ d=np=\Omega(\log n),\ K=2,$  $e_1\rightarrow 0 \Leftrightarrow s-\beta>0$; $e_2\rightarrow 0 \Leftrightarrow s+\beta s^2-\beta\sqrt{(\beta+s)(1+\beta s)}>0$; $e_3\rightarrow 0 \Leftrightarrow s+\beta s^2-\beta s-\beta^2>0,$ where $e_1,\ e_2,\ e_3$ are the average errors of the prediction when random walk, normalized Laplacian and Algorithm \ref{alg:Norm_sou} are applied respectively.
\end{theorem}

The following figure shows  the boundaries of the three algorithms described in the above theorem.

\begin{figure}[H] 
\includegraphics[width=.6\linewidth]{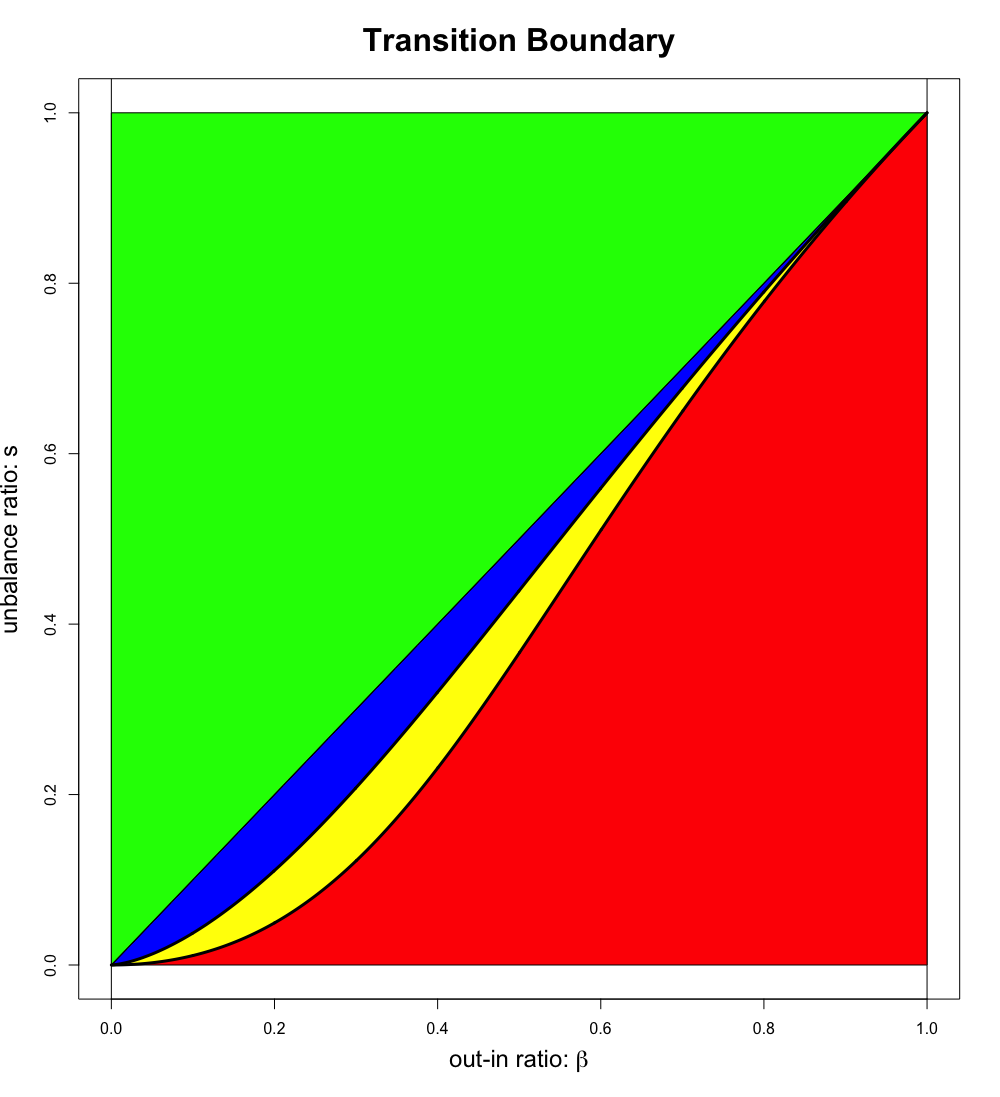}
\caption{All three algorithms can predict node labels properly in the green area of the graph. Normalized Laplacian can handle the blue area while random work cannot. However, the yellow area in the graph can only be solved by Algorithm \ref{alg:Norm_sou}. The red area is too unbalanced so that none of the three algorithms can perform better than random guess.}
\end{figure}
 
 Theorem \ref{boundary} indicates that Algorithm \ref{alg:Norm_sou} can tackle scenarios of greater imbalance than the other two methods, which is also observed in the empirical study. 

Although Algorithm \ref{alg:Norm_sou} gives sharper phase boundaries for unbalanced networks, in our simulation, Algorithm \ref{alg:Norm_targ} has its own advantage when applied to balanced networks. We speculate that it is because Algorithm \ref{alg:Norm_targ} does not rely on degree information as strongly as Algorithm \ref{alg:Norm_sou} does, so the former might perform better in balanced networks and degree-corrected networks,  where the degree information contains more noise than useful information. That is why we retain Algorithm \ref{alg:Norm_targ} in WIL hoping we can learn the proper $\rho$ from the data itself. $\rho$ is supposed to be a trade-off for the importance of degree's information.

\section{Simulation}\label{sect_simu}
\subsection{Network generation scheme, performance measure and default parameters}\label{modelsetting}

Throughout the simulation studies, we use the pDCBM to generate networks with 2,000 nodes and two communities. We follow the simulation scheme in \cite{r1}. The community labels of nodes are outcomes of independent multinomial draws with $\pi = (\pi_1, \pi_2)$. Conditional on these labels, the edges are generated as independent Bernoulli variables with $p=B_{c(i), c(j)}$, while under the heterogeneous setting, $p=\theta_i\theta_j B_{c(i), c(j)}$. We use $\theta_i$ to represent the popularity of node $i$ and $\theta_i$'s are drawn independently with $P(\theta = 0.2) = \gamma$ and $P(\theta = 1) = 1-\gamma$. We consider two settings, namely $\gamma = 0$ and $\gamma = 0.9$, which correspond to the homogeneous setting and the heterogeneous setting respectively.

The block probability matrix $B$ is determined by two parameters: the overall edge density $\lambda$ and the out-in ratio $\beta$. $\lambda$ is indeed $E(degree)$. It ranges from $2$ to $12$ in our simulations and a small $\lambda$ indicates a sparse network. $\beta$ determines the ratio of inter- to intra-community connection probabilities, and is set between $0.04$ and $0.4$. To generate $B$, we first generate $B^{(0)}$, whose diagonal and off-diagonal entries are set to $\beta^{-1}$ and $1$ respectively. Then $B^{(0)}$ is rescaled so that $E(degree)=\lambda$. Specifically,
\begin{equation}
B = \frac{\lambda}{(n-1)(\pi^T B^{(0)}\pi)(\mathbb{E}\theta)^2}B^{(0)}.
\end{equation}

$\rho$ selection: In following simulations and real-data analysis, we first select $\rho \in \mathcal{P} =  \{0, 0.1, 0.2, \cdots, 0.9, 1\}$. By comparing the computing accuracy of labeled nodes, we select the $\rho$ with the best performance as the parameter to predict the unlabeled nodes in networks. 
\begin{equation}
\rho = \text{argmax}_{\rho \in \mathcal{P}} Accuracy(\rho).
\end{equation}
We repeat the random sampling code, and visualized the average choice of $\rho$ in different settings. In addition, we can also apply other optimization ideas to obtain a better $\rho$ but we will not discuss this in detail in this paper.   

All of simulations below 
adopt the same network generation scheme as that described above. We control the parameters $\lambda$, $\beta$, $\gamma$ and $\pi$ to simulate different settings. Under each parameter setting, we replicate the simulation process 50 times (unless otherwise stated) and report the average performance of various methods. 

\subsection{Comparison of methods}
We carry out extensive simulations to compare the WIL methods with the cutting-edge methods including algorithms introduced for graph-based semi-supervised learning (GSSL) in the literature. 
The following methods/algorithms are adopted for comparisons:
\begin{itemize}
\item Partially absorbing random walks (PARW) \cite{parw}, 
\item Learning with local and global consistency (LGCiter) \cite{r608} 
\item Semi-supervised learning using Gaussian fields and harmonic functions (HMNiter) \cite{r609}
\item Confidence-aware modulated label propagation (CAMLP) \cite{camlp},
\item New regularized algorithms for transductive learning (MAD) \cite{r611} 
\end{itemize}
\subsubsection{Degree-homogeneous setting}
We start the simulations under the homogeneous setting ($\gamma = 0$). We run three groups of simulations to test the methods, varying $\lambda$, $\beta$ or $\pi$ in each group. Specifically, for $\pi$, we use $\pi = (1/2- \Delta,1/2 + \Delta)$ with $\Delta$ varying between 0 and 0.4. Here $\Delta$ can be interpreted as the degree of imbalance in community size.

Figures \ref{fig:sbm_lambda_95} and \ref{fig:sbm_lambda_90} show the performance of the methods, as the networks change from sparse to dense. Figures \ref{fig:sbm_oir_95} and \ref{fig:sbm_oir_90} show their performance in terms of the change in the out-in ratio. Generally speaking, a larger $\beta$ means a smaller "contrast" in the observed networks and therefore more difficult tasks. On the other hand, we also consider varying $\pi$, because imbalance in group size could be an issue in real applications. Figures \ref{fig:sbm_diff_95} and \ref{fig:sbm_diff_90} show the performance of the tested methods pertaining to this issue. Figures \ref{fig:sbm_rho_95} and \ref{fig:sbm_rho_90} show the corresponding average selection of $\rho$ when $\pi$ is changing.
\begin{figure}[H]
    \centering
    \begin{subfigure}[H]{0.45\textwidth}
        \includegraphics[trim = 0 0 20 0, clip, width=\textwidth]{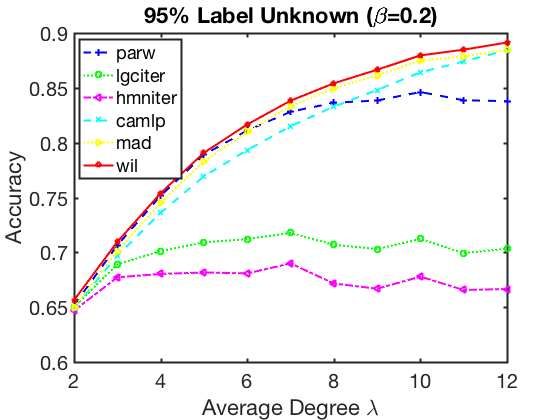}
        \caption{Accuracy with varying $\lambda$'s.}
        \label{fig:sbm_lambda_95}
    \end{subfigure}
    \begin{subfigure}[H]{0.45\textwidth}
        \includegraphics[trim = 0 0 20 0, clip, width=\textwidth]{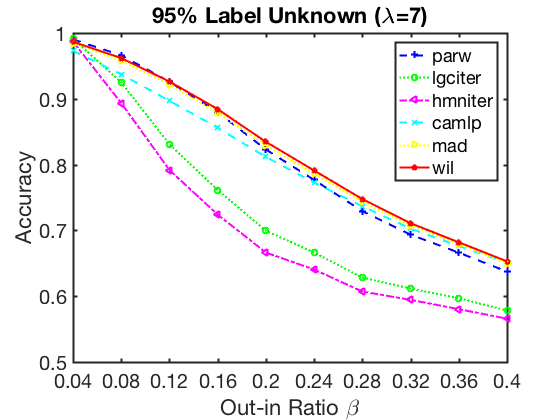}
        \caption{Accuracy with varying $\beta$'s.}
        \label{fig:sbm_oir_95}
    \end{subfigure}
\begin{subfigure}[H]{0.45\textwidth}
        \includegraphics[trim = 0 0 20 0, clip, width=\textwidth]{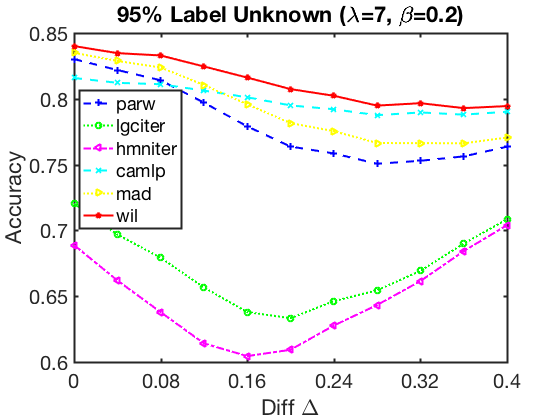}
        \caption{Accuracy with varying degrees of imbalance in community size.}
        \label{fig:sbm_diff_95}
    \end{subfigure}
    \begin{subfigure}[H]{0.45\textwidth}
        \includegraphics[trim = 0 0 20 0, clip, width=\textwidth]{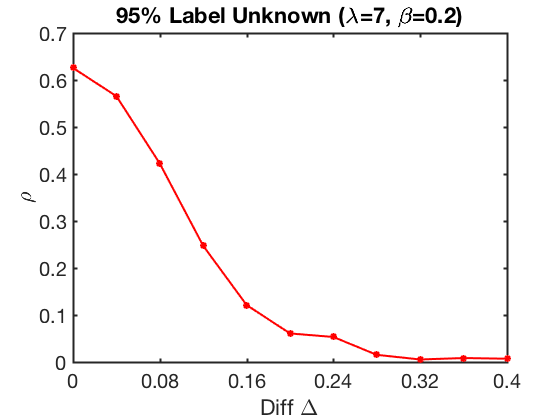}
        \caption{Average selection of $\rho$ with varying degrees of imbalance in community size.}
        \label{fig:sbm_rho_95}
    \end{subfigure}
    \caption{Comparison in the homogeneous setting with $95\%$ of the labels unknown: Networks are simulated from the pDCBM  with $n = 2000$, $K=2$, $\gamma=0$; in (a), $\beta = 0.2$ and $\pi = (1/2,1/2)$; in (b), $\lambda = 7$ and $\pi = (1/2,1/2)$; and in (c), $\lambda = 7, \beta = 0.2$, and $\pi = (1/2 -\Delta,1/2 + \Delta)$.}
\label{Fig:CompareSBM95}
\end{figure}

\begin{figure}[H]
    \centering
    \begin{subfigure}[H]{0.45\textwidth}
        \includegraphics[trim = 0 0 20 0, clip, width=\textwidth]{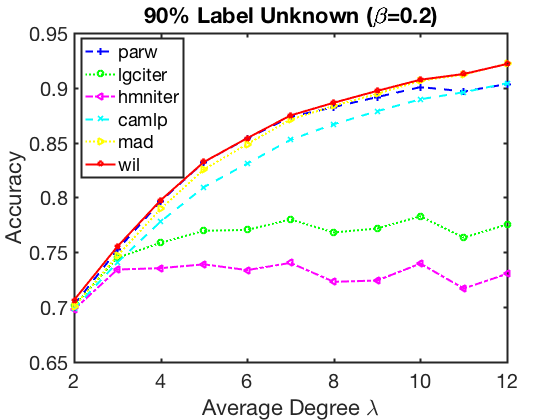}
        \caption{Accuracy with varying $\lambda$'s.}
        \label{fig:sbm_lambda_90}
    \end{subfigure}
    \begin{subfigure}[H]{0.45\textwidth}
        \includegraphics[trim = 0 0 20 0, clip, width=\textwidth]{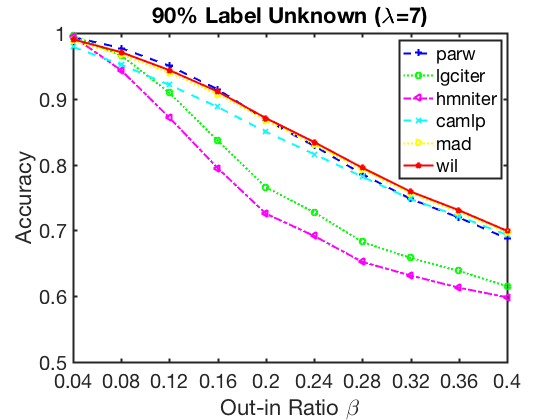}
        \caption{Accuracy with varying $\beta$'s.}
        \label{fig:sbm_oir_90}
    \end{subfigure}
\end{figure}
\begin{figure}[H]\ContinuedFloat
\begin{subfigure}[H]{0.45\textwidth}
        \includegraphics[trim = 0 0 20 0, clip, width=\textwidth]{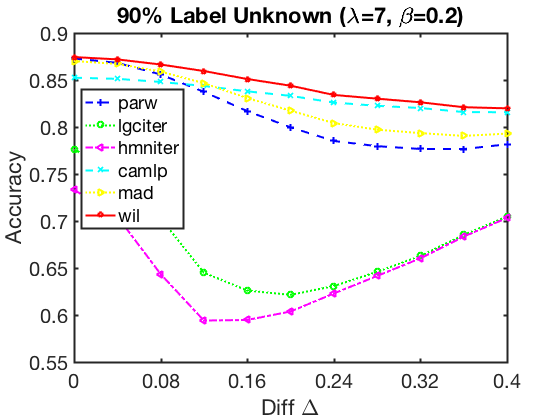}
        \caption{Accuracy with varying degrees of imbalance in community size.}
        \label{fig:sbm_diff_90}
    \end{subfigure}
    \begin{subfigure}[H]{0.45\textwidth}
        \includegraphics[trim = 0 0 20 0, clip, width=\textwidth]{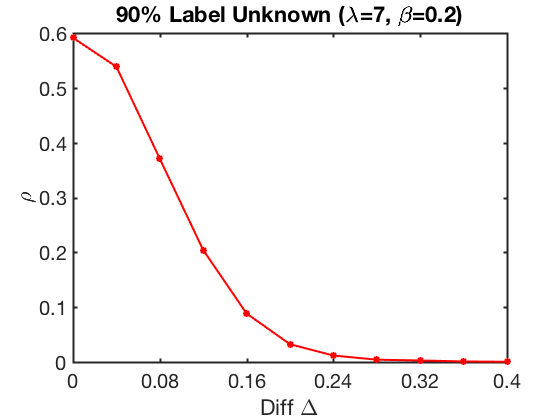}
        \caption{Average selection of $\rho$ with varying degrees of imbalance in community size.}
        \label{fig:sbm_rho_90}
    \end{subfigure}
    \caption{Comparison in the homogeneous setting with $90\%$ of the labels unknown: Networks are simulated from the pDCBM with $n = 2000$, $K=2$, $\gamma=0$; in (a), $\beta = 0.2$ and $\pi = (1/2,1/2)$; in (b), $\lambda = 7$ and $\pi = (1/2,1/2)$; and in (c), $\lambda = 7, \beta = 0.2$, and $\pi = (1/2 -\Delta,1/2 + \Delta)$.}
\label{Fig:CompareSBM90}
\end{figure}
Overall, the WIL method is more competitive in a more general setting, and it consistently ranks among the top methods. When the community sizes are inhomogeneous, the accuracy of WIL is among the best. At the same time, as the imbalance of community sizes increases, the average value of best choice of $\rho$ decreases, which confirms our theoretical analysis.

\subsubsection{Degree-heterogeneous setting}
We repeat the simulations above in the heterogeneous setting. The only difference here is that, in the generation of networks, 10\% of the nodes are hubs with high popularity. The results are shown in Figure \ref{Fig:CompareDCSBM90} and \ref{Fig:CompareDCSBM95}.

\begin{figure}[H]
    \centering
    \begin{subfigure}[H]{0.45\textwidth}
        \includegraphics[trim = 0 0 20 0, clip, width=\textwidth]{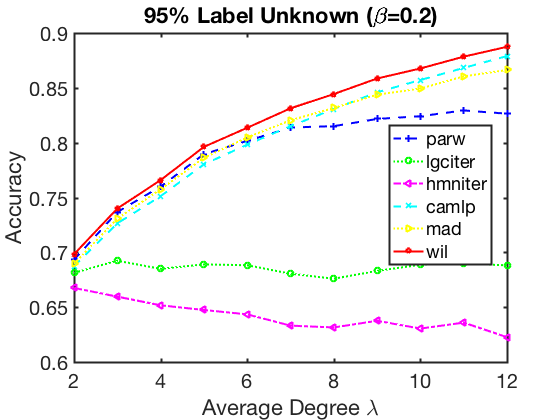}
        \caption{Accuracy with varying $\lambda$'s.}
        \label{fig:dcsbm_lambda_95}
    \end{subfigure}
    \begin{subfigure}[H]{0.45\textwidth}
        \includegraphics[trim = 0 0 20 0, clip, width=\textwidth]{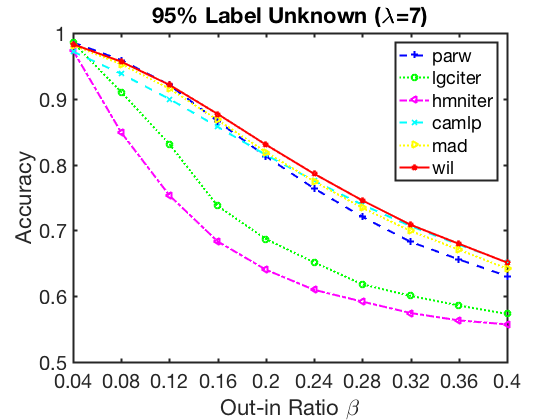}
        \caption{Accuracy with varying $\beta$'s.}
        \label{fig:dcsbm_oir_95}
    \end{subfigure}
\end{figure}
\begin{figure}[H]\ContinuedFloat
\begin{subfigure}[H]{0.45\textwidth}
        \includegraphics[trim = 0 0 20 0, clip, width=\textwidth]{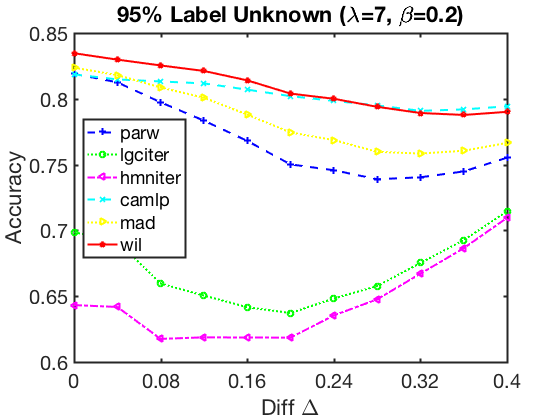}
        \caption{Accuracy with varying degrees of imbalance in community size.}
        \label{fig:dcsbm_diff_95}
    \end{subfigure}
    \begin{subfigure}[H]{0.45\textwidth}
        \includegraphics[trim = 0 0 20 0, clip, width=\textwidth]{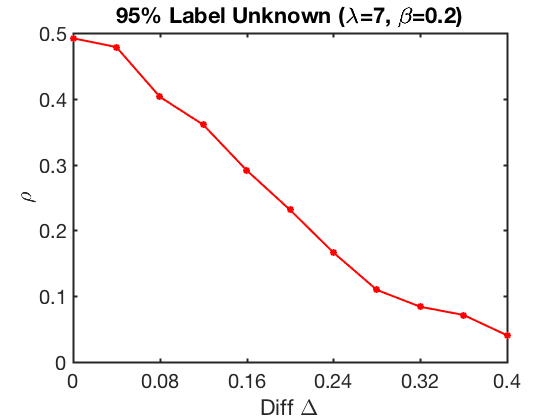}
        \caption{Average selection of $\rho$ with varying degrees of imbalance in community size.}
        \label{fig:dcsbm_rho_95}
    \end{subfigure}
    \caption{Comparison in the heterogeneous setting with $95\%$ of the labels unknown: Networks are simulated from the pDCBM with $n = 2000$, $K=2$, $\gamma=0.9$; in (a), $\beta = 0.2$ and $\pi = (1/2,1/2)$; in (b), $\lambda = 7$ and $\pi = (1/2,1/2)$; and in (c), $\lambda = 7, \beta = 0.2$, and $\pi = (1/2-\Delta,1/2 + \Delta)$.}
\label{Fig:CompareDCSBM95}
\end{figure}

\begin{figure}[H]
    \centering
    \begin{subfigure}[H]{0.45\textwidth}
        \includegraphics[trim = 0 0 20 0, clip, width=\textwidth]{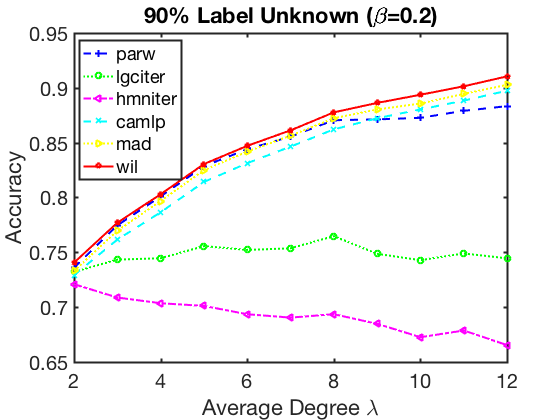}
        \caption{Accuracy with varying $\lambda$'s.}
        \label{fig:dcsbm_lambda_90}
    \end{subfigure}
    \begin{subfigure}[H]{0.45\textwidth}
        \includegraphics[trim = 0 0 20 0, clip, width=\textwidth]{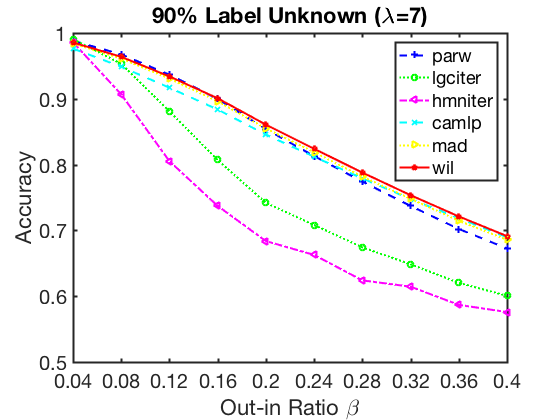}
        \caption{Accuracy with varying $\beta$'s.}
        \label{fig:dcsbm_oir_90}
    \end{subfigure}
\begin{subfigure}[H]{0.45\textwidth}
        \includegraphics[trim = 0 0 20 0, clip, width=\textwidth]{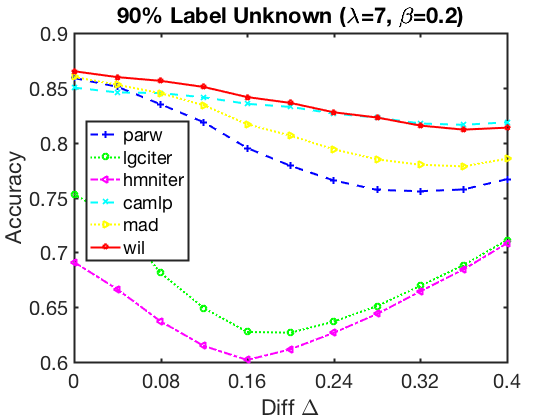}
        \caption{Accuracy with varying degrees of imbalance in community size.}
        \label{fig:dcsbm_diff_90}
    \end{subfigure}
    \begin{subfigure}[H]{0.45\textwidth}
        \includegraphics[trim = 0 0 20 0, clip, width=\textwidth]{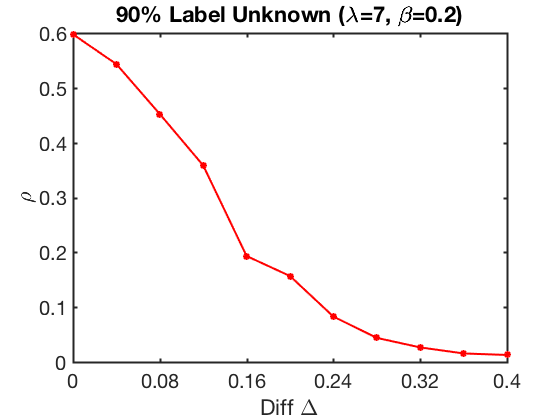}
        \caption{Average selection of $\rho$ with varying degrees of imbalance in community size.}
        \label{fig:dcsbm_rho_90}
    \end{subfigure}
    \caption{Comparison in the heterogeneous setting with $90\%$ of the labels unknown: Networks are simulated from the pDCBM with $n = 2000$, $K=2$, $\gamma=0.9$; in (a), $\beta = 0.2$ and $\pi = (1/2,1/2)$; in (b), $\lambda = 7$ and $\pi = (1/2,1/2)$; and in (c), $\lambda = 7, \beta = 0.2$, and $\pi = (1/2 -\Delta,1/2 + \Delta)$.}
\label{Fig:CompareDCSBM90}
\end{figure}
Similar to the homogeneous setting, the WIL method is more competitive in a more general heterogeneous setting, and it consistently ranks among the top methods. When the community sizes are inhomogeneous, the accuracy of WIL is among the best. At the same time, when the imbalance of community sizes increases, the average value of the best choice of $\rho$ decreases, which confirms our theoretical analysis again.

\section{Real-data Analysis}\label{sect_real}
In this section, we examine the performance of the WIL algorithm with real network data. Those methods considered in the simulations above are applied here as well. Three commonly studied datasets are used.

\begin{itemize}
\item[] {\bf Political blog network} \cite{r0} is regarded as a typical degree-corrected network \cite{r4}. The data were collected immediately after the 2004 US presidential election. Pairs of blogs are connected if there is a hyperlink between them. The giant component of it contains 1,222 blogs and 16,714 edges, where each blog is manually labeled as either liberal or conservative. The belief that blogs with similar political attitudes tend to be connected makes this network ideal for network community studies. Many researchers have tested their methods on this dataset to see how close their results of community detection are to the manual labels. 

\item[] {\bf Facebook friendship network} The Facebook network dataset consists of all the "friendship" links between users within each of 100 US universities, recorded in 2005. The dataset contains several node attributes such as the gender, dorm, graduation year, and academic major of each user. 
\begin{itemize}
\item[] {\bf Facebook Simmons college network (Simmons)} The Simmons College Facebook network is a friendship network that contains
1,518 nodes and 32,988 undirected relationship edges. We followed common pre-processing steps by considering the largest connected component of the students with graduation years (from 2006 to 2009; 4 communities), which leads to a subgraph of 1,137 nodes and 24,257 edges.
It was observed in \cite{facebook} that the class year had the highest assortativity values among all available
demographic characteristics, and so we treated the class year as the true community label.

\item[] {\bf Facebook Caltech network (Caltech)} Different from the Simmons College network in which communities are formed according to class years, communities in the Caltech friendship network are recorded by dorms \cite{facebook}. By using dorms as labels, we also treated students spread across eight different dorms as true community labels. Following the same pre-processing steps, we excluded the students whose residence information was missing and considered the largest connected component of the remaining network, which contained 590 nodes and 12,822 undirected edges. This dataset with more label kinds is more challenging than the Simmons College dataset. 

\end{itemize}
\end{itemize}

Table \ref{datadesc} shows a summary of the three datasets and Figure \ref{Fig:realhist} and \ref{Fig:realdeg} report the details of community size distribution and degree distribution. From the distributions, we can see that the community sizes are not ideally balanced in real networks. Moreover, the distributions of the degree are also quite different. These findings illustrate why we need to design and analyze the network propagation algorithm under general settings.

\begin{table}[H]
\centering
\begin{tabular}{|c|c|c|c|}
\hline
                 & n (number of nodes) & K (number of communities) & average degree \\ \hline
Political blogs    & 1222                & 2                         & 27.36          \\ \hline
Facebook (Simmons)    & 1137                 & 4                         & 42.67          \\ \hline
Facebook (Caltech) & 590                 & 8                        & 43.46           \\ \hline
\end{tabular}
\caption{Real-data description}
\label{datadesc}
\end{table}

\begin{figure}[H]
    \centering
    \begin{subfigure}[H]{0.32\textwidth}
        \includegraphics[trim = 0 0 20 0, clip, width=\textwidth]{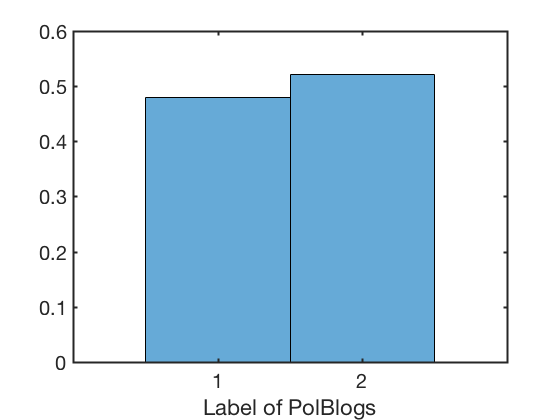}
        \caption{Hist. of Political Blogs.}
        \label{fig:realhistpolblogs}
    \end{subfigure}
    \begin{subfigure}[H]{0.32\textwidth}
        \includegraphics[trim = 0 0 20 0, clip, width=\textwidth]{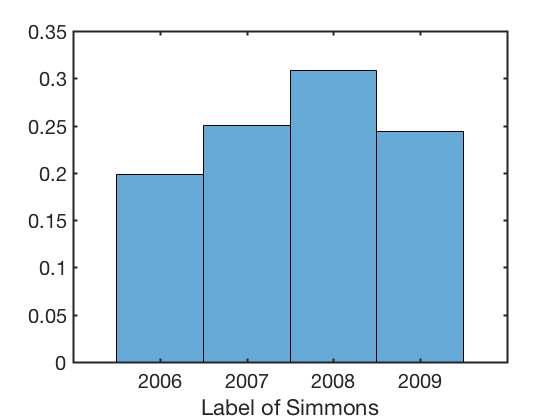}
        \caption{Hist. of Simmons.}
        \label{fig:realhistsimmons}
    \end{subfigure}
\begin{subfigure}[H]{0.32\textwidth}
        \includegraphics[trim = 0 0 20 0, clip, width=\textwidth]{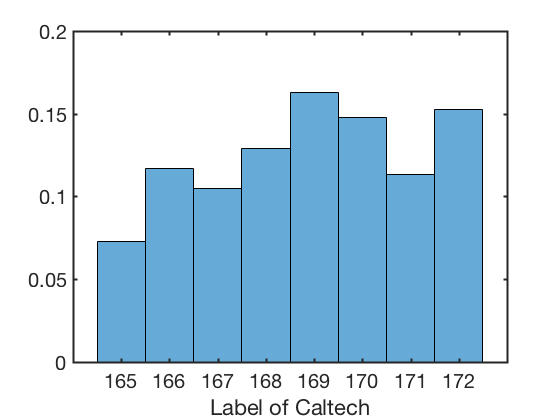}
        \caption{Hist. of Caltech.}
        \label{fig:realhistcaltech}
    \end{subfigure}
    \caption{Histogram of real-data}
\label{Fig:realhist}
\end{figure}

\begin{figure}[H]
    \centering
    \begin{subfigure}[H]{0.32\textwidth}
        \includegraphics[trim = 0 0 20 0, clip, width=\textwidth]{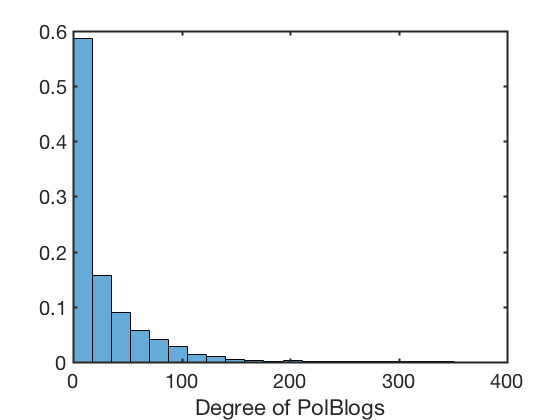}
        \caption{Degree of Political Blogs.}
        \label{fig:realdegpolblogs}
    \end{subfigure}
    \begin{subfigure}[H]{0.32\textwidth}
        \includegraphics[trim = 0 0 20 0, clip, width=\textwidth]{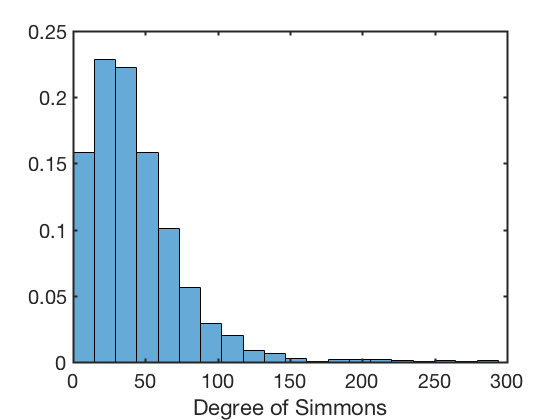}
        \caption{Degree of Simmons.}
        \label{fig:realdegsimmons}
    \end{subfigure}
\begin{subfigure}[H]{0.32\textwidth}
        \includegraphics[trim = 0 0 20 0, clip, width=\textwidth]{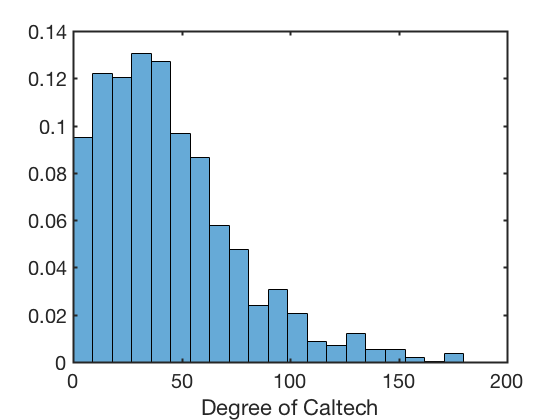}
        \caption{Degree of Caltech.}
        \label{fig:realdegcaltech}
    \end{subfigure}
    \caption{Degree distribution of real-data}
\label{Fig:realdeg}
\end{figure}

 Figures \ref{fig:realpolblogs}-\ref{fig:realcaltech} report the performance of the considered methods. The set of methods examined here is the same as that used in the simulation section.

\begin{figure}[H]
    \centering
    \begin{subfigure}[H]{0.48\textwidth}
        \includegraphics[trim = 0 0 20 0, clip, width=\textwidth]{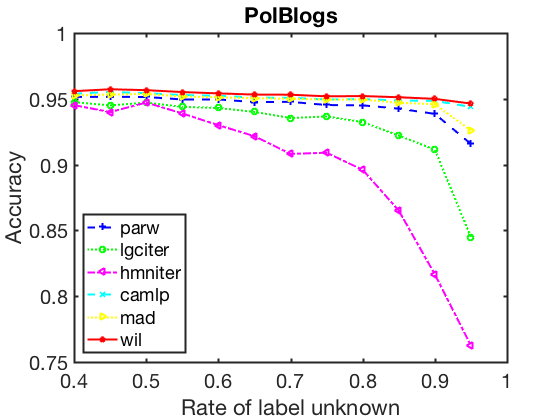}
        \caption{Test of Political Blogs.}
        \label{fig:realpolblogs}
    \end{subfigure}
\end{figure}
\begin{figure}[H]\ContinuedFloat
    \begin{subfigure}[H]{0.48\textwidth}
        \includegraphics[trim = 0 0 20 0, clip, width=\textwidth]{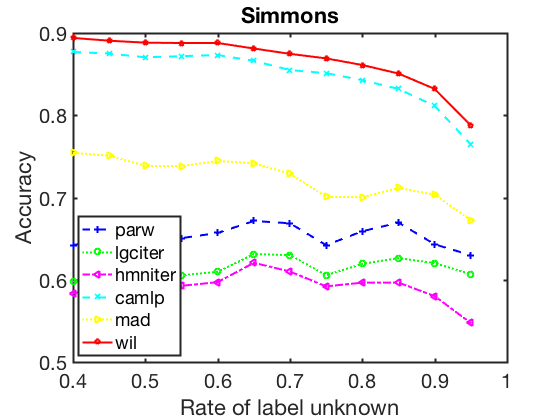}
        \caption{Test of Facebook Simmons.}
        \label{fig:realsimmons}
    \end{subfigure}
\begin{subfigure}[H]{0.48\textwidth}
        \includegraphics[trim = 0 0 20 0, clip, width=\textwidth]{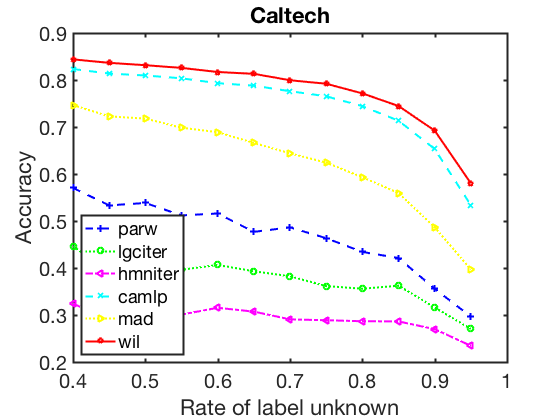}
        \caption{Test of Facebook Caltech.}
        \label{fig:realcaltech}
    \end{subfigure}
    \caption{Accuracy with varying unknown rates.}
\label{Fig:realtest}
\end{figure}

\begin{figure}[H]
    \centering
        \includegraphics[trim = 0 0 20 0, clip, width=0.5\textwidth]{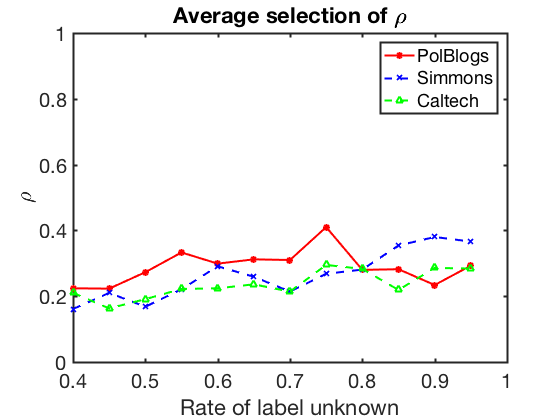}
        \caption{Average selection of $\rho$ with varying unknown rates.}
\label{Fig:realrho}
\end{figure}

The WIL gives very competitive results in different settings. Especially, when the percentage of labeled data decreases, the accuracy of WIL performance ranks among the best. What is also worth noticing here is the choice of $\rho$ in real-data analysis. In Figure \ref{Fig:realrho}, we recorded the average selection of $\rho$. From the performance, we can see that the choices of $\rho$ are small. This gives one of reason why the traditional random walk fails in these test cases.

\section{Related Works}\label{sect_related_works}
Graph-based semi-supervised learning (GSSL) is a well-studied topic in computer science and engineering. \cite{r605} first used min-cuts in a graph to perform clustering and \cite{r606} introduced normalized min-cuts to deal with the issue of unbalanced networks. Other spectral-based methods were later proposed, including  \cite{r607, r608}. A Gaussian kernel was used to construct the graph in \cite{r609}. A random walk based method called Adsorption was proposed in \cite{r610} and was modified in \cite{r611} into something called MAD. TACO, which was proposed in \cite{r612}, introduced an additional quantity of confidence in labeling. We recommend \cite{r613} for a good review of GSSL. While many methods on GSSL were developed in the last decade, few considered the consistency theoretically. \cite{r614} is the only paper we found that discussed the consistency of the basic GSSL method theoretically.

Although some methods have proven to be efficient at GSSL, they might not be able to perform network-based semi-supervised learning (NSSL) well. First, in NSSL, we obtain the network directly, so the network structure is unclear, while the network for GSSL is always constructed by a similarity measure. Second, there is randomness in the link generation of networks, which implies more noise and less information, thus increasing the difficulty of the NSSL problem. Last, networks for GSSL are almost fully connected, while those for NSSL might be very sparse which will cause heterogeneity. 

\section{Conclusion and discussioin}\label{sect_CONC}

We proposed a scalable method called WIL for semi-supervised learning in networks and a new generative model called the pDCBM for the problem . The underlying idea of WIL is to enhance the information represented by an adjacency matrix by considering the combination of two random walks with different normalizations in the network. This method is designed specifically for unbalanced networks, although it works well for balanced networks as well. It also works superbly when the network is heterogeneous. Both theoretical study and empirical study show the advantage of the WIL algorithm for heterogeneous networks. 

It would be interesting to study the theoretical properties of the WIL algorithm under the sparse network setting in the future. In this paper, for the dense scenario ($E(degree)=\Omega(\log n)$), we have proved the consistency of WIL. However, simulations suggest that denseness is not required in practice. Therefore, it would be interesting to explore in a theoretical study if we could extend the result to $E(degree)=\Omega(1)$. Additionally, the pDCBM can also be extended, for example,  by letting $K$ goes to infinity with $n$, and making matrix $B$ more general.  

More over, WIL might work for other problems such as the regression problem in networks. We can also extend the network generating model to weighted edges instead of 0/1 as well as directed networks. We leave all of these open problems to feature research.

\section{Acknowledgments}
The authors would like to thank Prof. Zhigang Bao for helpful advice.

\appendix
\section{Technical Proof}
We prove the main results here by introducing useful notation first. For any matrix (vector) $M$, $[M]_i$ denotes the $i'th$ row of $M.$ Let $P=\Theta ZBZ^T\Theta^T$, and $d_i=\sum_{j}p_{i,j}$ which is the expected degree of node $i$. We set $D$ to be the diagonal matrix with $d_{i,i}=d_i.$ 

We give the proof of main results based on the pDCBM. First, let us make the network homogeneous, which means $\theta_i=1 \ \forall i \in {1,2,\dots,n},$ and $K=2.$ When $K=2,$ we transform $Y$ into a vector, $[Y]_i=1$ if $i \in L \and\ c(i)=1$, $[Y]_i=-1 $ if $i \in L \and\ c(i)=2,$ and $[Y]_i=0$ if $c(i)\in U.$ 

Let us discuss the behavior of degrees first. The following result is Lemma 8 in \cite{r5}:

\begin{lemma}\label{var_of_d}
Let $\delta_{i,c}=max\{d_i, c\log n\}$. With probability $1-2/n^{c_1-1},$ one has 
$$\|\hat{d}_i-d_i\|\leq c_2 \sqrt{\delta_{i,c} \log n} \ \  for \  each \  i=1,\dots, n,$$
where $c_1=0.5c_2^2/(1+c_2/\sqrt{c}).$
\end{lemma}

From the above lemma, we arrive at the following result immediately:

\begin{lemma}\label{lemma_degree}
When $min\{d_i\}=\Omega(\log n),$ for any $k\in \mathbb{N},$ there exists a constant $c\geq 0.$ With probability $1-n^{-c},$
$$[(\hat{D}^{-1}A)^kY]_i=[(D^{-1}A)^kY]_i+o(1),$$
$$[(A \hat{D}^{-1})^kY]_i=[(AD^{-1})^kY]_i+o(1).$$
\end{lemma}
\begin{proof}
From Lemma \ref{var_of_d}, with probability $1-2/n^{c_1-1},$ we have 
$$\max_i|\frac{\hat{d}_i}{d_i}-1|\leq \max_i c_2\sqrt{\delta_{i,c} \log n} /d_i\leq max_i c_2\sqrt{\frac{\log n}{d_i}}\rightarrow 0,$$ 
which means $\hat{d}_i \rightarrow d_i \ in \ probability, \  for \  each \ i=1,\dots n.$ Without loss of generality, we set $c(i)=1.$ We have:
\begin{align*}
& |[(D^{-1}A)^kY]_i-[(\hat{D}^{-1}A)^kY]_i |\\
& = |\sum_{j_1,\dots j_k \in G}(-1)^{c(j_k)-1}a_{i,j_1}a_{j_1,j_2}\dots a_{j_{k-1},j_k}(\frac{1}{d_id_{j_1}\dots d_{j_{k-1}}}-\frac{1}{\hat{d}_i\hat{d}_{j_1}\dots \hat{d}_{j_{k-1}}})|\\
& \leq \sum_{j_1,\dots j_k \in G}a_{i,j_1}a_{j_1,j_2}\dots a_{j_{k-1},j_k}o(\frac{1}{\hat{d}_i\hat{d}_{j_1}\dots \hat{d}_{j_{k-1}}})\\
& = \sum_{j_1,\dots j_{k-1} \in G}a_{i,j_1}a_{j_1,j_2}\dots a_{j_{k-2},j_{k-1}}o(\frac{1}{\hat{d}_i\hat{d}_{j_1}\dots \hat{d}_{j_{k-1}}})\sum_{j_k\in G}a_{j_{k-1},j_k}\\
& = \sum_{j_1,\dots j_{k-1} \in G}a_{i,j_1}a_{j_1,j_2}\dots a_{j_{k-2},j_{k-1}}o(\frac{1}{\hat{d}_i\hat{d}_{j_1}\dots \hat{d}_{j_{k-2}}})\\
&\dots\\
& = \sum_{j_1\in G}a_{i,j_1}o(\frac{1}{\hat{d}_i})\\
& = o(1)
\end{align*}
\end{proof}

Now, we mainly focus on $[(D^{-1}A)^kY]_i$ and $[(AD^{-1})^kY]_i.$ We first look at their expectations.

\begin{lemma}\label{lemma_expect}
When $min\{d_i\}=\Omega(\log n),$ for any $k\in \mathbb{N}$,
$$\mathbb{E}([(D^{-1}A)^kY]_i)=[(D^{-1}P)^kY]_i+O(\frac{1}{d}),$$
$$\mathbb{E}([(AD^{-1})^kY]_i)=[(PD^{-1})^kY]_i+O(\frac{1}{d}).$$
\end{lemma}
\begin{proof}
Without loss of generality, we assume $c(i)=1.$
\begin{equation}
\mathbb{E}([(D^{-1}A)^kY]_i)=\sum_{j_1,\dots j_k \in G}(-1)^{c(j_k)-1}\frac{\mathbb{E}(a_{i,j_1}a_{j_1,j_2}\dots a_{j_{k-1},j_k})}{d_id_{j_1}\dots d_{j_{k-1}}}
\end{equation}
If $a_{i,j_1}, a_{j_1,j_2},\dots, a_{j_{k-1},j_k}$ are all independent, then\\
$\mathbb{E}(a_{i,j_1} a_{j_1,j_2}\dots a_{j_{k-1},j_k})=p_{i,j_1}p_{j_1,j_2}\dots p_{j_{k-1},j_k}.$ If there exist only $k-m$ independent random variables in  $a_{i,j_1},a_{j_1,j_2},\dots, a_{j_{k-1},j_k},$ then $\{j_1,j_2,\dots j_k\}$ has $k-m$ different values at most, which means $|\{j_1,j_2,\dots j_k\}|\leq k-m.$ Their sum is:\\
 $\sum_{|\{j_1,j_2,\dots j_k\}|\leq k-m} \frac{\mathbb{E}(a_{i,j_1}a_{j_1,j_2}\dots a_{j_{k-1},j_k})}{d_id_{j_1}\dots d_{j_{k-1}}}\leq    \left(
    \begin{array}{c}
      k \\
      m
    \end{array}
  \right) 
  (k-m)^m
   \left(
    \begin{array}{c}
      n \\
      k-m
    \end{array}
  \right)p^{k-m}/d^k=O(\frac{1}{d^m}).$
So \begin{equation}
\mathbb{E}([(D^{-1}A)^kY]_i)=\sum_{j_1,\dots j_k \in G}(-1)^{c(j_k)-1}\frac{p_{i,j_1}p_{j_1,j_2}\dots p_{j_{k-1},j_k}}{d_id_{j_1}\dots d_{j_{k-1}}}+O(\frac{1}{d})=[(D^{-1}P)^kY]_i+O(\frac{1}{d}).
\end{equation}
Similarly, we can prove $\mathbb{E}([(AD^{-1})^kY]_i)=[(PD^{-1})^kY]_i+O(\frac{1}{d}).$
\end{proof}

As for variance, we have the following lemma:
\begin{lemma} \label{lemma_var}
When $min\{d_i\}=\Omega(\log n),$ for any $k\in \mathbb{N}$,
$$Var([(\hat{D}^{-1}A)^kY]_i)=O(\frac{1}{d}),$$
$$Var([(A\hat{D}^{-1})^kY]_i)=O(\frac{1}{d}).$$
\end{lemma}
\begin{proof}
Same as the proof of Lemma \ref{lemma_expect}, we can write
\begin{align*}
& Var([(\hat{D}^{-1}A)^kY]_i)\\
& = \mathbb{E}(\sum_{j_1,\dots j_k \in G}(-1)^{c(j_k)-1}\frac{a_{i,j_1}a_{j_1,j_2}\dots a_{j_{k-1},j_k}}{\hat{d}_i\hat{d}_{j_1}\dots \hat{d}_{j_{k-1}}})^2-[\mathbb{E}(\sum_{j_1,\dots j_k \in G}(-1)^{c(j_k)-1}\frac{a_{i,j_1}a_{j_1,j_2}\dots a_{j_{k-1},j_k}}{\hat{d}_i\hat{d}_{j_1}\dots \hat{d}_{j_{k-1}}})]^2\\
& =\sum_{|j_1,\dots j_k , l_1,\dots l_k|=2k} \mathbb{E}(\frac{a_{i,j_1}a_{j_1,j_2}\dots a_{j_{k-1},j_k}a_{i,l_1}a_{j_1,l_2}\dots a_{l_{k-1},l_k}}{\hat{d}_i\hat{d}_{j_1}\dots \hat{d}_{j_{k-1}}\hat{d}_i\hat{d}_{l_1}\dots \hat{d}_{l_{k-1}}})-\\
& \sum_{|j_1,\dots j_k , l_1,\dots l_k|=2k}\mathbb{E}(\frac{a_{i,j_1}a_{j_1,j_2}\dots a_{j_{k-1},j_k}}{\hat{d}_i\hat{d}_{j_1}\dots \hat{d}_{j_{k-1}}})\mathbb{E}(\frac{a_{i,l_1}a_{j_1,l_2}\dots a_{l_{k-1},l_k}}{\hat{d}_i\hat{d}_{l_1}\dots \hat{d}_{l_{k-1}}})
+O(\frac{1}{d})\\
& =\sum_{|j_1,\dots j_k , l_1,\dots l_k|=2k} p_{i,j_1}p_{j_1,j_2}\dots p_{j_{k-1},j_k}p_{i,l_1}p_{j_1,l_2}\dots p_{l_{k-1},l_k}[\\
&\mathbb{E}(\frac{1}{\hat{d}_i\hat{d}_{j_1}\dots \hat{d}_{j_{k-1}}\hat{d}_i\hat{d}_{l_1}\dots \hat{d}_{l_{k-1}}}|a_{i,j_1}=a_{j_1,j_2}=\dots= a_{j_{k-1},j_k}=a_{i,l_1}=a_{j_1,l_2}=\dots =a_{l_{k-1},l_k}=1)-\\
&\mathbb{E}(\frac{1}{\hat{d}_i\hat{d}_{j_1}\dots \hat{d}_{j_{k-1}}}|a_{i,j_1}=a_{j_1,j_2}=\dots= a_{j_{k-1},j_k}=1)\mathbb{E}(\frac{1}{\hat{d}_i\hat{d}_{l_1}\dots \hat{d}_{l_{k-1}}}|a_{i,l_1}=a_{j_1,l_2}=\dots =a_{l_{k-1},l_k}=1)]\\
&+O(\frac{1}{d})
\end{align*}
We rewrite 
$$\frac{1}{X_j}=\frac{1}{\hat{d}_i\hat{d}_{j_1}\dots \hat{d}_{j_{k-1}}}|a_{i,j_1}=a_{j_1,j_2}=\dots= a_{j_{k-1},j_k}=1$$
and 
$$\frac{1}{X_l}=\frac{1}{\hat{d}_i\hat{d}_{l_1}\dots \hat{d}_{l_{k-1}}}|a_{i,l_1}=a_{j_1,l_2}=\dots =a_{l_{k-1},l_k}=1.$$
If we can prove 
$$\mathbb{E}(\frac{1}{X_j}\frac{1}{X_l})-\mathbb{E}(\frac{1}{X_j})\mathbb{E}(\frac{1}{X_l})=O(\frac{1}{d^{2k+1}})$$
then we will get 
$$Var([(\hat{D}^{-1}A)^kY]_i)=d^{2k}O(\frac{1}{d^{2k+1}})+O(\frac{1}{d})=O(\frac{1}{d}),$$ 
which will complete the proof.\\
Indeed, when $k=1$, applying the Taylor expansion yields
\begin{align*}
&\mathbb{E}(\frac{1}{X_j}\frac{1}{X_l})-\mathbb{E}(\frac{1}{X_j})\mathbb{E}(\frac{1}{X_l})\\
&=\mathbb{E}\frac{1}{\hat{d}_i}\frac{1}{\hat{d}_i}-\mathbb{E}\frac{1}{\hat{d}_i}\mathbb{E}\frac{1}{\hat{d}_i}\\
&=Var(\frac{1}{\hat{d}_i})\\
&=\mathbb{E}(\frac{1}{\mathbb{E}(\hat{d}_i^2)}-\frac{1}{\mathbb{E}(\hat{d}_i^2)}(\hat{d}_i^2-\mathbb{E}(\hat{d}_i^2))+\dots)-(\mathbb{E}(\frac{1}{d_i}-\frac{1}{d_i}(\hat{d}_i-d_i)+\dots))^2\\
&=\frac{1}{\mathbb{E}(\hat{d}_i^2)}-\frac{1}{d_i^2}+ O(\frac{1}{d^3})\\
&=O(\frac{1}{d^3}).
\end{align*}
When $k\geq 2$, from Lemma \ref{var_of_d}, with probability $1-n^{-c}$, we have
\begin{align*}
&\mathbb{E}(\frac{1}{X_j}\frac{1}{X_l})-\mathbb{E}(\frac{1}{X_j})\mathbb{E}(\frac{1}{X_l})\\
&=O(\frac{1}{d^{2k-2}})(\mathbb{E}\frac{1}{\hat{d}_{j_k}}\frac{1}{\hat{d}_{l_k}}-\mathbb{E}\frac{1}{\hat{d}_{j_k}}\mathbb{E}\frac{1}{\hat{d}_{l_k}})\\
&=O(\frac{1}{d^{2k-2}})cov(\frac{1}{\hat{d}_{j_k}},\frac{1}{\hat{d}_{l_k}})\\
&\leq O(\frac{1}{d^{2k-2}})\sqrt{var(\frac{1}{\hat{d}_{j_k}})var(\frac{1}{\hat{d}_{l_k}})}\\
&= O(\frac{1}{d^{2k-2}})O(\frac{1}{d^3})\\
&=O(\frac{1}{d^{2k+1}}).
\end{align*}

Similarly, we can prove that $Var([(A\hat{D}^{-1})^kY]_i)=O(\frac{1}{d}).$
\end{proof}

Finally, we can compute the value of $[(D^{-1}P)^kY]_i$ and $[(PD^{-1})^kY]_i$ as follows:

\begin{lemma} \label{lemma_det}
For any matrix in a block form:
$$
W=
\begin{pmatrix}
A & B\\
C & D
\end{pmatrix},
$$
where elements in each block are all the same, $A \ is \ an\  n_1 \times n_1$ matrix and $D \ is  \ an \ n_2 \times n_2$ matrix. Let $Y$ be an $n_1+n_2$ dimensional vector with $Y_i=1, \ i=1,2, \dots, l_1$, $Y_j=-1, \ j=(n_1+1),(n_1+2), \dots, (n_1+l_2)$ and other elements set to $0$, where $l_1<n_1$ and $l_2<n_2.$ If  $I- W$ is invertible, then 
$$[(I-W)^{-1}Y]_i=\frac{al_1+bcn_2l_1-bl_2-adn_1l_2}{(1-n_1a)(1-n_2b)-n_1n_2bc},$$ 
where $i \in \{(l_1+1),(l_1+2), \dots, n_1\}$ a, b, c, and d are the values of elements in A, B, C, and D respectively.
\end{lemma}
\begin{proof}
Let $F=(I-W)^{-1}Y.$ Then we have $(I-W)F=Y$ from Cramer's rule:
$$F_i=\frac{det((I-W)_{*i})}{det(I-W)},$$ where $(I-W)_{*i}$ means replace the $i-th$ column of $(I-W)$ by vector $Y.$ Calculations yield
$$ det((I-W)_{*i})=al_1+bcn_2l_1-bl_2-adn_1l_2$$ and
$$det(I-W)=(1-n_1a)(1-n_2b)-n_1n_2bc.$$
\end{proof}

From Lemma \ref{lemma_det}, $\forall i \in U$ and $c(i)=1,$ we obtain the following corollary:
%
%
%
%
\begin{corollary}\label{cor_eq}
\begin{equation} \label{eq_rw}
 [(I-\alpha D^{-1}P)^{-1}Y]_i=\frac{\alpha \delta((s-\beta)(1+\beta s)-\alpha s (1-\beta^2))}{(1-\alpha)((\beta+s)(1+\beta s)-\alpha s(1-\beta))} .
\end{equation}

\begin{equation} \label{eq_nl}
 [(I-\alpha D^{-1/2}PD^{-1/2})^{-1}Y]_i=\frac{\alpha \delta(s+\beta s^2-\beta \sqrt{(\beta+s)(1+\beta s)}-\alpha s (1-\beta^2))}{(1-\alpha)((\beta+s)(1+\beta s)-\alpha s(1-\beta))}.
\end{equation}

\begin{equation} \label{eq_irw}
 [(I-\alpha PD^{-1})^{-1}Y]_i=\frac{\alpha \delta(s+\beta s^2-\beta s-\beta^2-\alpha s(1-\beta^2))}{(1-\alpha)((\beta+s)(1+\beta s)-\alpha s(1-\beta))}.
\end{equation}

\end{corollary}

 Now, we have the following Lemma:
 
 \begin{lemma} \label{lemma_point}
Under the pDCBM setting, with $\theta_i=1\ \forall i \in \{1,2,\dots,n\}$, $d=np=\Omega (\log n)$, $K=2$, and $s >  g(\beta)$, $\forall i \in U$ there exist some constant $c >0,$ such that 
$$\mathbb{P}(\hat{c}(i) \neq c(i))\leq \frac{c}{(1-\beta)^2s^2\delta^2d},$$
where $\hat{c}(i)$ is the prediction by applying WIL, $g(\beta)=\frac{1}{2\beta }((\beta-1)+\sqrt{4\beta ^3+\beta^2-2\beta +1})$, $0< \beta=\frac{q}{p} <1$ and  $0< s=min\{\frac{n_1}{n_2},\frac{n_2}{n_1}\}\leq 1.$
\end{lemma}
\begin{proof}
First, let us consider one node $i\in U.$ Without loss of generality, we assume $c(i)=1.$ Let $\Delta_i=[(I-\alpha A\hat{D}^{-1})^{-1}Y]_i.$ Following Lemmas \ref{lemma_degree}, \ref{lemma_expect} and \ref{lemma_det}, we have
\begin{align*}
\mathbb{E}(\Delta_i)  & =[(I-\alpha PD^{-1})^{-1}Y]_i+\frac{\alpha}{1-\alpha} o(1)\\
                                       & =\frac{\alpha \delta}{(1-\alpha)((\beta+s)(1+\beta s)-\alpha s(1-\beta))}(s+\beta s^2-\beta s-\beta^2-\alpha s(1-\beta^2))+\frac{\alpha}{1-\alpha} o(1)\\
\end{align*}
From Lemma \ref{lemma_var}, we get
\begin{align*}
Var(\Delta_i)  & \leq  \left (\sum_j\sqrt{Var([\alpha^j(A\hat{D}^{-1})^j]_i)}\right )^2\\
                         & =\left(\frac{\alpha}{1-\alpha}\right)^2O\left(\frac{1}{d}\right).
\end{align*}
When $\mathbb{E}(\Delta_i )>0,$ from Chebyshev's inequality, we have
\begin{align*}
\mathbb{P}(\hat{c}(i) \neq c(i))=\mathbb{P}(\Delta_i \leq 0)  &\leq \left (\frac{\sigma_i}{\mathbb{E}(\Delta_i)}\right )^2\\
                                                   &=\left(\frac{(\beta+s)(1+\beta s)-\alpha s(1-\beta)}{\delta(s+\beta s^2-\beta s-\beta^2-\alpha s(1-\beta^2))}\right)^2O\left(\frac{1}{d}\right)\\
                                                   & = \frac{c}{\delta^2d}\left(\frac{(\beta+s)(1+\beta s)}{s+\beta s^2-\beta s-\beta^2}\right)^2\\
                                                    & = \frac{c}{\delta^2d}\left(\frac{\beta+s}{s+\beta s^2-\beta s-\beta^2}\right)^2\\
                                                    & = \frac{c}{\delta^2d}\left(\frac{1}{s(1-\beta)}\right)^2
\end{align*}
\end{proof}
Under the homogeneity assumption, Theorem \ref{homo} can be proved easily by using Lemma \ref{lemma_point}.
 \begin{proof}
 $$\mathbb{E}(err)\leq \max_{i\in U}\mathbb{P}(\Delta_i \leq 0) \leq \frac{c}{(1-\beta)^2s^2\delta^2d}.$$
$$Var(err)\leq \max_{i \in U}\{Var(1_{\hat{c}(i)\neq c(i)})\} \leq \frac{c}{(1-\beta)^2s^2\delta^2d}. $$
Using Chebyshev's inequality, we have:
$$\mathbb{P}(err \geq \epsilon )\leq \frac{c}{\epsilon ^2(1-\beta)^2s^2\delta^2d}.$$
 \end{proof}

Now, we extend the above result to heterogeneous networks (degree-corrected networks). The only difference is introduced by $\Theta=(\theta_1,\theta_2,\dots,\theta_n)^T.$ Just like most work on the DCBM (\cite{r9, r600, r601, r4}), we treat $\Theta$ as given. Now the link probability matrix becomes $P=\Theta ZBZ^T\Theta^T.$ For the identity issue, we assume $\max_i\{\theta_i\}=1;$ otherwise, we can rewrite $\Theta=\Theta/\max_i\{\theta_i\}$ and $B=B\max_i\{\theta_i\}$ instead.  This is also mentioned in \cite{r9}.  In \cite{r600}, the authors assume $\min_i\{\theta_i\} \geq c_0,$ where $c_0$ is a constant. In \cite{r4}, although only $0 < \min_i\{\theta_i\}$ is required, the author assumes the expectation of degree to be a polynomial of $n.$ The expectation degree of node $i$ is $d_i=\sum_{j}\theta_i\theta_jp_{i,j}.$ In order to keep Lemma \ref{lemma_degree}'s result, we need $\min\{d_i\}=\Omega(\log n),$ or $\min\{\theta_inp\}=\Omega(\log n),$ which is slightly looser than the restriction mentioned before. When a node's popularity is too low,  the linkage is too sparse to carry useful information, making the prediction much harder.

Since $\theta_i\leq 1,$ and  $\min\{d_i\}=\Omega(\log n),$ it is easy to prove that Lemma \ref{lemma_expect} and Lemma \ref{lemma_var} both hold.

Last, we also need $\forall u \in [K], \frac{1}{n_u}\sum_{c(i)=u}\theta_i\in [1-\delta_1,1],$ where $\delta_1=o(1),$ which is also proposed in \cite{r601}, and $\frac{1}{l_u}\sum_{c(i)=u, i\in L}\theta_i\in [1-\delta_1,1]$ to keep Lemma \ref{lemma_det}. This restriction means that groups should be similar to each other in terms of overall popularity; otherwise, the most popular group will absorb more nodes, which will lead to great bias.

So far, we have proved Theorem \ref{homo}. Theorem \ref{boundary} can be proved similarly by replacing the kernel.

\begin{proof}
From Corollary \ref{cor_eq}, we can see that if the three inequalities hold, the result can be proved following the proof of Theorem \ref{homo}. 

While the inequality does not hold, the consistency of prediction cannot be reached. We take consistency of $e_3$ as an example. The other two kernels can be proved similarly.

Let $s+\beta s^2-\beta s-\beta^2\leq 0$ and without loss of generality set $n_1<n_2.$ For any node $i\in U$ with $c(i)=1,$ we have $\mathbb{E}(\Delta_i)\leq 0$ from Lemma \ref{lemma_expect}. So we have 
\begin{align*}
\mathbb{P}(\hat{c}(i) = c(i))=\mathbb{P}(\Delta_i \geq 0)  &\leq \left (\frac{\sigma_i}{-\mathbb{E}(\Delta_i)}\right )^2\\
                                                    & = \frac{c}{\delta^2d}\left(\frac{1}{s(1-\beta)}\right)^2,
\end{align*}
which means we almost incorrectly predicted the label of $i.$ So with probability $1-\frac{c}{d},$ we have
$$e_3 \geq \frac{\beta}{1+\beta},$$
which is the error rate by just predicting all nodes with the same label what has more members (here we predict that all labels are equal to $2$).
\end{proof}

\section{Extension to general K}

When $K$ is larger than $2$, as long as $K$ is a constant, we can use the comparison idea to reach the outcome. Since we only assign the index of the highest score as the label of the node, we can compare the scores of different labels in a pairwise manner. The main part of proof will not change, while the determinant value in Lemma \ref{lemma_det} will need to be recalculated. Since it is only a calculation issue, we just give a rough result. Roughly, we should divide the result in Lemma \ref{lemma_det} by $K-1$. So the convergence rate becomes $$\mathbb{P}(err \geq \epsilon )\leq \frac{c(K-1)^2}{\epsilon ^2(1-\beta)^2s^2\delta^2d}.$$

\end{document}